\documentclass[twoside]{article}
\usepackage[accepted]{aistats2020}
\usepackage[round]{natbib} 

\usepackage[utf8]{inputenc} 
\usepackage[T1]{fontenc}    

\usepackage{xcolor}
\definecolor{mydarkblue}{rgb}{0,0.08,0.45}
\usepackage[colorlinks=true,
    linkcolor=mydarkblue,
    citecolor=mydarkblue,
    filecolor=mydarkblue,
    urlcolor=mydarkblue]{hyperref} 

\usepackage{eucal} 

\usepackage{amsfonts}
\usepackage{nicefrac}       
\usepackage{amsmath}
\usepackage{amssymb}
\usepackage{cancel}
\usepackage{mathtools}
\usepackage{bbm}

\usepackage{url}            
\usepackage{booktabs}       
\usepackage{microtype}      

\usepackage{enumitem}

\usepackage{cleveref}
\crefformat{section}{\S#2#1#3}

\usepackage{tikz}
\usetikzlibrary{bayesnet}

\usepackage{caption}
\usepackage{subcaption}
\usepackage{graphicx}
\usepackage{tikz}
\usetikzlibrary{arrows}
\graphicspath{{imgs/}}
\usepackage{pgfplots}
\usepackage{float}

\usepackage{amsthm}
\newtheorem{theorem}{Theorem}

\newtheorem{lemma}{Lemma}
\newtheorem{conjecture}{Conjecture}

\newtheorem{definition}{Definition}
 
\theoremstyle{remark}

\usepackage[colorinlistoftodos]{todonotes}

\newcounter{todocounter}

\newcommand{\reals}{\mathbb{R}}

\DeclarePairedDelimiter{\norm}{\lVert}{\rVert}

\newcommand{\lpar}{\left(}
\newcommand{\rpar}{\right)}

\newcommand{\lspar}{\left[}
\newcommand{\rspar}{\right]}

\newcommand{\diff}{\, \text{d}}

\newcommand{\vect}[1]{\mathbf{#1}}

\newcommand{\mat}[1]{\mathbf{#1}}
\newcommand{\eye}{\mathbf{I}}
\newcommand{\mJ}{\mat{J}}

\newcommand{\vp}{\vect{p}}
\newcommand{\vq}{\vect{q}}

\newcommand{\mK}{\mat{K}}

\newcommand{\mLam}{\mat{\Lambda}}
\newcommand{\mThe}{\mat{\Theta}}

\newcommand{\Kp}{\mK \vp}
\newcommand{\Kq}{\mK \vq}

\newcommand{\KOEnt}{\mathbb{H}_{1}^\mK}
\newcommand{\KEnt}{\mathbb{H}_{\alpha}^\mK}

\newcommand{\BDiv}{\mathbb{D}^\mK}

\newcommand{\X}{\mathcal{X}}
\newcommand{\Y}{\mathcal{Y}}

\newcommand{\Xs}{(\mathcal{X}, \kappa)}

\newcommand{\vx}{\vect{x}}

\newcommand{\Spx}{\mathbf{\Delta}_n}

\newcommand{\Prob}{\mathbb{P}}
\newcommand{\Qrob}{\mathbb{Q}}

\newcommand{\Exp}{\mathbb{E}}

\newcommand{\Ent}{\mathbb{H}}

\DeclareMathOperator*{\Uniform}{\mathcal{U}}

\begin{document}
	
	%
	
	%
	
	\twocolumn[
	
	\runningauthor{Gallego, Vani, Schwarzer, Lacoste-Julien}
	\aistatstitle{GAIT: A Geometric Approach to Information Theory}
	\aistatsauthor{Jose~Gallego \hspace{5mm} Ankit~Vani \hspace{5mm} Max~Schwarzer \hspace{5mm}  Simon~Lacoste-Julien$^\dagger$}
	\aistatsaddress{ Mila and DIRO, Université de Montréal}
	
	]

	\begin{abstract}
		
		We advocate the use of a notion of entropy that reflects the relative abundances of the symbols in an alphabet, as well as the similarities between them. This concept was originally introduced in theoretical ecology to study the diversity of ecosystems. Based on this notion of entropy, we introduce geometry-aware counterparts for several concepts and theorems in information theory. Notably, our proposed divergence exhibits performance on par with state-of-the-art methods based on the Wasserstein distance, but enjoys a closed-form expression that can be computed efficiently. We demonstrate the versatility of our method via experiments on a broad range of domains: training generative models, computing image barycenters, approximating empirical measures and counting modes.
		
		
		
	\end{abstract}
	
	
	
	\vspace{-2ex}

\section{Introduction}
Shannon's seminal theory of information \citeyearpar{shannon} has been of paramount importance in the development of modern machine learning techniques. However, standard information measures deal with probability distributions over an alphabet considered as a mere set of symbols and disregard additional geometric structure, which might be available in the form of a metric or similarity function. As a consequence of this, information theory concepts derived from the Shannon entropy (such as cross entropy and the Kullback-Leibler divergence) are usually blind to the geometric structure in the domains over which the distributions are defined.

This blindness limits the applicability of these concepts. For example, the Kullback-Leibler divergence cannot be optimized for empirical measures with non-matching supports. Optimal transport distances, such as Wasserstein, have emerged as practical alternatives with theoretical grounding. These methods have been used to compute barycenters~\citep{cuturi_fast_bar} and train generative models~\citep{cuturi_learning}. However, optimal transport is computationally expensive as it generally lacks closed-form solutions and requires the solution of linear programs or the execution of matrix scaling algorithms, even when solved only in approximate form~\citep{cuturi_regularized}. Approaches based on kernel methods \citep{ak2st, mmd_gan, ot_gan}, which take a functional analytic view on the problem, have also been widely applied. However, further exploration on the interplay between kernel methods and information theory is lacking. 

\textbf{Contributions.} We \textit{i)} introduce to the machine learning community a similarity-sensitive definition of entropy developed by~\citet{leinster2012}. Based on this notion of entropy we \textit{ii)} propose geometry-aware counterparts for several information theory concepts. We \textit{iii)} present a novel notion of divergence which incorporates the geometry of the space when comparing probability distributions, as in optimal transport. However, while the former methods require the solution of an optimization problem or a relaxation thereof via matrix-scaling algorithms, our proposal enjoys a closed-form expression and can be computed efficiently.  We refer to this collection of concepts as Geometry-Aware Information Theory: \textit{GAIT}.


\textbf{Paper structure.} We introduce the theory behind the GAIT entropy and provide motivating examples justifying its use. We then introduce and characterize a divergence as well as a definition of mutual information derived from the GAIT entropy. Finally, we demonstrate applications of our methods including training generative models, approximating measures and finding barycenters. We also show that the GAIT entropy can be used to estimate the number of modes of a probability distribution.

\begin{figure*}[t]
    \vspace{-2mm}
	\begin{minipage}{.32\textwidth}
		\centering
		\vspace{-2mm}
		\includegraphics[trim=5 17 5 40, clip, width=0.93\textwidth]{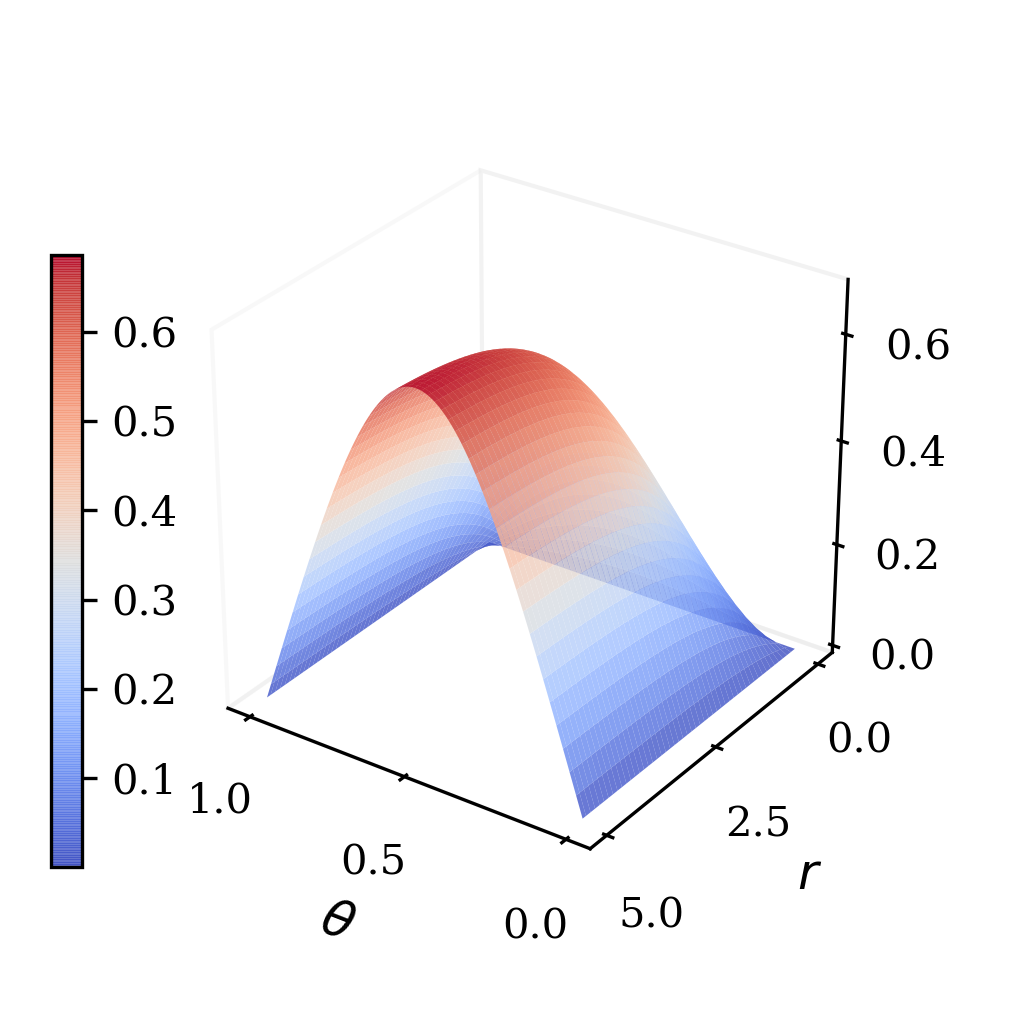}
		\captionof{figure}{$\mathbb{H}^\mK_1$ interpolates towards the Shannon entropy as $r \rightarrow \infty$.}
		\label{fig:ent_2d}
	\end{minipage}
	\hspace{7mm}
	\begin{minipage}{.27\textwidth}
		\vspace{8mm}
		\hspace{2mm}
		\begin{tikzpicture}[<->,>=stealth',shorten >=1pt,auto,node distance=2cm, scale=0.95, every node/.style={scale=0.85}, thick, main node/.style={circle, draw}]
		
		\node[main node] (a) {A};
		\node[main node] (b) [below of=a] {B};
		\node[main node] (c) [above right=0.5cm and 2.1cm of b] {C};
		
		\path
		(a) edge [loop left] node {1} (a)
		edge node [left] {0.7} (b)
		edge node {0.1} (c)
		(b) edge [loop left] node {1} (b)
		(c) edge [loop above] node {1} (c)
		edge node {0.1} (b);
		\end{tikzpicture}
		\vspace{6mm}
		\captionof{figure}{A 3-point space with two highly similar elements.}
		\label{fig:3pt_space}
	\end{minipage}
	\hfill
	\begin{minipage}{.3\textwidth}
		\centering
		\vspace{5mm}
		\includegraphics[width=0.93\textwidth]{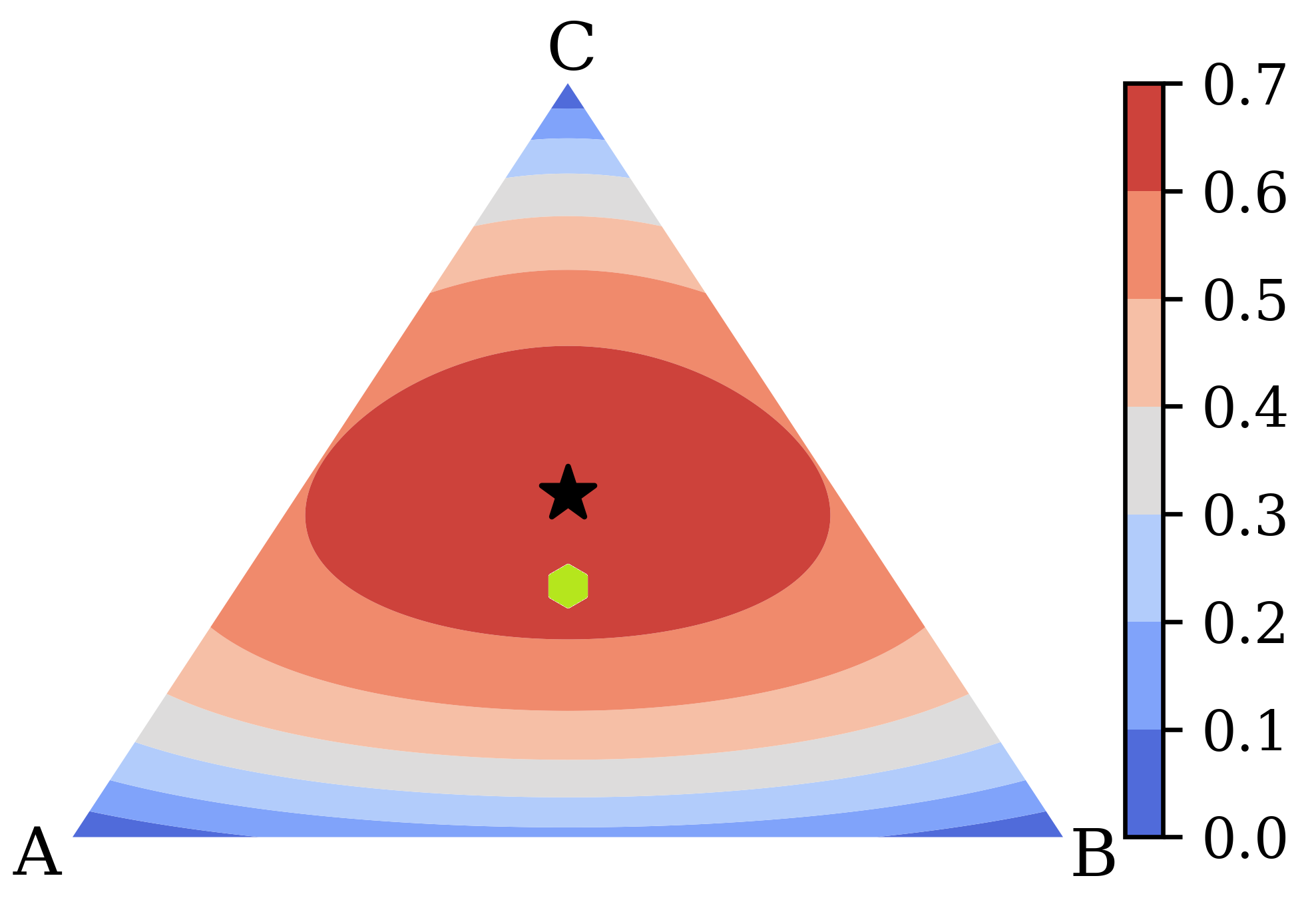}
		\captionof{figure}{$\mathbb{H}^\mK_1$ for distributions over the space in Fig. \ref{fig:3pt_space}.}
		\label{fig:toy_heatmap}
	\end{minipage}
\end{figure*}

\textbf{Notation.}
Calligraphic letters denote $\mathcal{S}$ets, bold letters represent $\vect{M}$atrices and $\vect{v}$ectors, and double-barred letters denote $\Prob$robability distributions and information-theoretic functionals. To emphasize certain computational aspects, we alternatively  denote a distribution $\Prob$ over a finite space $\X$ as a vector of probabilities $\vp$. $\eye$, $\mathbf{1}$ and $\vect{J}$ denote the identity matrix, a vector of ones and matrix of ones, with context-dependent dimensions.
For vectors $\vect{v}$, $\vect{u}$ and $\alpha \in \reals$, $\frac{\vect{v}}{\vect{u}}$ and $\vect{v}^\alpha$ denote element-wise division and exponentiation.  $\left< \cdot, \cdot \right>$ denotes the Frobenius inner-product between two vectors or matrices. $\Spx \triangleq \{ \vect{x} \in \reals^n | \left< \mathbf{1}, \vect{x} \right> = 1 \text{ and  } x_i \ge 0 \}$ denotes the probability simplex over $n$ elements. $\delta_x$ denotes a Dirac distribution at point $x$. We adopt the conventions $0 \cdot \log(0) = 0$ and $x \log(0) = -\infty$ for $x > 0$. 

\textbf{Reproducibility.} Our experiments can be reproduced via: 
\url{https://github.com/jgalle29/gait}

	\section{\hspace{-2mm}Geometry-Aware Information Theory}
\label{sec:theory}

Suppose that we are given a finite space $\mathcal{X}$ with $n$ elements along with a symmetric function that measures the similarity between elements, $\kappa: \X \times \X \to [0, 1]$. Let $\mK$ be the matrix induced by $\kappa$ on $\X$; i.e, $\mK_{x,y} \triangleq \kappa_{xy} \triangleq \kappa(x, y) = \kappa(y, x)$. $\mK_{x,y} = 1$ indicates that the elements $x$ and $y$ are identical, while $\mK_{x,y} = 0$ indicates full dissimilarity. We assume that $\kappa(x, x) = 1$ for all $x \in \X$. We call $\Xs$ a (finite) similarity space. For brevity we denote $\Xs$ by $\X$ whenever $\kappa$ is clear from the context. 

Of particular importance are the similarity spaces arising from metric spaces. Let $(\X, d)$ be a metric space and define $\kappa(x, y) \triangleq e^{- d(x, y)}$. Here, the symmetry and range conditions imposed on $\kappa$ are trivially satisfied. The triangle inequality in $(\X, d)$ induces a multiplicative transitivity on $\Xs$: for all $x, y, z \in \X$, $\kappa(x, y) \ge \kappa(x, z) \kappa(z, y)$. Moreover, for any metric space of the negative type, the matrix of its associated similarity space is positive definite~\citep[Lemma 2.5]{Reams1999}. 


In this section, we present a theoretical framework which quantifies the ``diversity'' or ``entropy'' of a probability distribution defined on a similarity space, as well as a notion of divergence between such distributions. 

\subsection{Entropy and diversity}

Let $\Prob$ be a probability distribution on $\X$.  $\Prob$ induces a \emph{similarity profile} $\mK \Prob : \X \to [0,1]$, given by $\mK \Prob(x) \triangleq \Exp_{y \sim \Prob} \lspar \kappa(x, y) \rspar = (\Kp)_x$.\footnote{This denotes the $x$-th entry of the result of the matrix-vector multiplication $\Kp$.}  $\mK \Prob(x)$ represents the expected similarity between element $x$ and a random element of the space sampled according to $\Prob$. Intuitively, it assesses how ``satisfied'' we would be by selecting $x$ as a one-point summary of the space. In other words, it measures the ordinariness of $x$, and thus $\frac{1}{\mK \Prob(x)}$ is the rarity or \emph{distinctiveness} of $x$ \citep{leinster2012}. Note that the distinctiveness depends crucially on both the similarity structure of the space and the probability distribution at hand.

Much like the interpretation of Shannon's entropy as the expected surprise of observing a random element of the space, we can define a notion of diversity as \emph{expected distinctiveness}: ${ \scriptstyle \sum_{x \in \X} \Prob(x)} \frac{1}{\mK \Prob(x)}$. This arithmetic weighted average is a particular instance of the family of power (or H\"older) means. Given $\vect{w} \in \Spx$ and $\vx \in \reals^n_{\ge 0}$, the \emph{weighted power mean of order $\beta$} is defined as ${ M_{\vect{w}, \beta}(\vx) \triangleq \left< \vect{w}, \vx^\beta \right> ^\frac{1}{\beta}}$. Motivated by this averaging scheme, \cite{leinster2012} proposed the following definition:

\begin{definition}
	\label{def:entropy} \citep{leinster2012}  \emph{(\textbf{GAIT Entropy})}
	The GAIT entropy of order $\alpha \ge 0$ of distribution $\Prob$ on finite similarity space $\Xs$ is given by:
	\vspace{-3ex}
	\begin{align}
	\label{eq:entropy}
	\hspace{0mm}\KEnt[\Prob] &\triangleq \log M_{\vp, 1 - \alpha}\lpar \frac{1}{\Kp} \rpar \\
	 &= \frac{1}{1-\alpha} \log \sum_{i = 1}^n  \vp_i \frac{1}{(\Kp)^{1-\alpha}_i}.
	\end{align}	
\end{definition}

It is evident that whenever $\mK = \eye$, this definition reduces to the R\'enyi entropy \citep{Renyi1961}. Moreover, a continuous extension of Eq. \eqref{eq:entropy} to $\alpha =1$ via a L'H\^opital argument reveals a similarity-sensitive version of Shannon's entropy: 
\begin{equation}
 \mathbb{H}^{\mK}_1[\Prob] = - \left< \vp,  \log  (\mK \vp) \right>  = - \Exp_{x \sim \Prob}[\log (\mK \Prob)_x ].   
\end{equation}

Let us dissect this definition via two simple examples. First, consider a distribution $\vp_
\theta = [\theta, 1-\theta]^\top$ over the points $\{x, y\}$ at distance $r \ge 0$, and define the similarity $\kappa_{xy} \triangleq e^{-r}$. As the points get further apart, the Gram matrix $\mK_r$ transitions from $\mJ$ to $\eye$. Fig. \ref{fig:ent_2d} displays the behavior of $\mathbb{H}^{\mK_r}_1[\vp_\theta]$. We observe that when $r$ is large we recover the usual shape of Shannon entropy for a Bernoulli variable. In contrast, for low values of $r$, the curve approaches a constant zero function. In this case, we regard both elements of the space as identical: no matter how we distribute the probability among them, we have low uncertainty about the qualities of random samples. Moreover, the exponential of the maximum entropy, ${ \exp \lspar \sup_{\theta} \mathbb{H}^{\mK_r}_1[\vp_\theta] \rspar = 1 + \tanh(r) \in [1, 2]}$, measures the \emph{effective number of points} \citep{leinster2016} at scale $r$.

Now, consider the space presented in Fig. \ref{fig:3pt_space}, where the edge weights denote the similarity between elements. The maximum entropy distribution in this space following Shannon's view is the uniform distribution $\vect{u} = [\frac{1}{3}, \frac{1}{3}, \frac{1}{3}]^\top$. This is counter-intuitive when we take into account the fact that points A and B are very similar. We argue that a reasonable expectation for a maximum entropy distribution is one which allocates roughly probability $\frac{1}{2}$ to point C and the remaining mass in equal proportions to points A and B. Fig. \ref{fig:toy_heatmap} displays the value of $\mathbb{H}^{\mK}_1$ for all distributions on the 3-simplex. The green dot represents $\vect{u}$, while the black star corresponds to the maximum GAIT entropy with [A, B, C]-coordinates $\vp^* \triangleq [0.273, 0.273, 0.454]^\top$. The induced similarity profile is $\mK \vp^* = [\frac{1}{2}, \frac{1}{2}, \frac{1}{2}]^\top$. Note how Shannon's probability-uniformity gets translated into a constant similarity profile.

\textbf{Properties.} We now list several important properties satisfied by the GAIT entropy, whose proofs and formal statements are contained in \citep{leinster2012} and \citep{leinster2016}:
\vspace{-1ex}
\begin{itemize}
	\item \textbf{Range}: $0 \le \KEnt[\Prob] \le \log(|\X|)$.
	\item $\mK$\textbf{-monotonicity}: Increasing the similarity reduces the entropy. Formally, if  $\kappa_{xy} \ge \kappa'_{xy}$ for all $x, y  \in \X$, then $\mathbb{H}^{\mJ}_\alpha[\Prob] \le \KEnt[\Prob] \le \mathbb{H}^{\mK'}_\alpha[\Prob] \le \mathbb{H}^{\eye}_\alpha[\Prob]$.
	\item \textbf{Modularity}: If the space is partitioned into fully dissimilar groups, $\Xs = \bigotimes_{c=1}^C (\X_c, \kappa_c)$, so that $\mK$ is a block matrix ($x \in \X_c, y \in \X_{c'}, c \neq c' \Rightarrow \kappa_{xy} = 0$), then the entropy of a distribution on $\X$ is a weighted average of the block-wise entropies.
	\item \textbf{Symmetry}: Entropy is invariant to relabelings of the elements, provided that the rows of $\mK$ are permuted accordingly.
	\item \textbf{Absence}: The entropy of a distribution $\Prob$ over $\Xs$ remains unchanged when we restrict the similarity space to the support of $\Prob$.
	\item \textbf{Identical elements}: If two elements are identical (two equal rows in $\mK$), then combining them into one and adding their probabilities leaves the entropy unchanged.
	\item \textbf{Continuity}: $\KEnt[\Prob]$ is continuous in $\alpha \in [0, \infty]$ for fixed $\Prob$, and continuous in $\Prob$ (w.r.t. standard topology on $\mathbf{\Delta}$) for fixed $\alpha \in (0, \infty)$. 
	\item $\alpha$-\textbf{Monotonicity}: $\KEnt[\Prob]$ is non-increasing in  $\alpha$. 
\end{itemize}

\textbf{The role of} $\alpha$\textbf{.} Def. \ref{def:entropy} establishes a family of entropies indexed by a non-negative parameter $\alpha$, which determines the \textit{relative importance of rare elements versus common ones}, where rarity is quantified by $\frac{1}{\mK \Prob}$. In particular, $\mathbb{H}^{\mK}_0[\Prob] = \log \left< \vp , \frac{1}{\mK \vp} \right>$. When $\mK = \eye$, $\mathbb{H}^{\mK}_0[\Prob] = \log |\text{supp}(\Prob)|$, which values rare and common species equally, while $ \mathbb{H}^{\mK}_\infty[\Prob] = - \log \max_{i \in \text{supp}(\vp)} (\Kp)_i$ only considers the most common elements. Thus, in principle, the problem of finding a maximum entropy distribution depends on the choice of $\alpha$.

\begin{theorem}
	\label{thm:max_ent}
	\citep{leinster2016} Let $\Xs$ be a similarity space. There exists a probability distribution $\Prob^*_\X$ that maximizes $\KEnt[\cdot]$ for all $\alpha \in \reals_{\ge 0}$, simultaneously. Moreover, $\Ent^*_\X \triangleq \hspace{-2mm} \underset{\Prob \in \mathbf{\Delta}_{|\X|}}{\sup} \hspace{-0mm} \KEnt[\Prob]$ does not depend on $\alpha$.
\end{theorem}


Remarkably, Thm. \ref{thm:max_ent} shows that the maximum entropy distribution is independent of $\alpha$ and thus, the maximum value of the GAIT entropy is an intrinsic property of the space: this quantity is a \textit{geometric invariant}. In fact, if $\kappa(x, y) \triangleq e^{-d(x, y)}$ for a metric $d$ on $\X$, there exist deep connections between $\Ent^*_\X$ and the magnitude of the metric space $(\X, d)$~\citep{leinster_magnitude}.

\begin{theorem}
	\label{thm:cst_sim_prof}
	~\citep{leinster2016} Let $\Prob$ be a distribution on a similarity space $\Xs$. $\KEnt[\Prob]$ is independent of $\alpha$ if and only if $\mK \Prob(x) = \mK \Prob(y)$ for all $x, y \in \text{supp}(\Prob)$.
\end{theorem}

Recall the behavior of the similarity profile observed for $\vp^*$ in Fig. \ref{fig:3pt_space}. Thm. \ref{thm:cst_sim_prof} indicates that this is not a coincidence: inducing a similarity profile which is constant over the support of a distribution $\Prob$ is a necessary condition for $\Prob$ being a maximum entropy distribution. In the setting $\alpha=1$ and $\mK = \eye$, the condition $\mK \vp = \vp = \lambda \mathbf{1}$ for some $\lambda \in \reals_{\ge 0}$, is equivalent to the well known fact that the uniform distribution maximizes Shannon entropy. 

\subsection{Concavity of $\mathbb{H}^{\mK}_1[\cdot]$}

A common interpretation of the entropy of a probability distribution is that of the amount of \textit{uncertainty} in the values/qualities of the associated random variable. From this point of view, the concavity of the entropy function is a rather intuitive and desirable property: ``entropy should increase under averaging''. 

Consider the case $\mK = \mathbf{I}$. $\mathbb{H}^{\mathbf{I}}_{\alpha}[\cdot]$ reduces to the the R\'enyi entropy of order $\alpha$. For general values of $\alpha$, this is not a concave function, but rather only Schur-concave~\citep{renyi_schur}. However, $\mathbb{H}^{\mathbf{I}}_1[\cdot]$ coincides with the Shannon entropy, which is a strictly concave function. Since the subsequent theoretical developments make extensive use of the concavity of the entropy, we restrict our attention to the case $\alpha=1$ for the rest of the paper. 

To the best of our knowledge, whether the entropy $\mathbb{H}^{\mK}_1[\Prob]$ is a (strictly) concave function of $\Prob$ for general similarity kernel $\mK$ is currently an open problem. Although a proof of this result has remained elusive to us, we believe there are strong indicators, both empirical and theoretical, pointing towards a positive answer. We formalize these beliefs in the following conjecture:
\begin{conjecture} Let $(\mathcal{X}, \kappa)$ be a finite similarity space with Gram matrix $\mK$. If $\mK$ is positive definite and $\kappa$ satisfies the multiplicative triangle inequality, then $\mathbb{H}_1^\mK[\cdot]$ is strictly concave in the interior of $\mathbf{\Delta}_{|\X|}$.  \label{conj:conc}
\end{conjecture}

\vspace{-0.7ex}

Fig. \ref{fig:concave_and_mog} shows the relationship between the linear approximation of the entropy and the value of the entropy over segment of the convex combinations between two measures. This behavior is consistent with our hypothesis on the concavity of $\mathbb{H}_1^\mK[\cdot]$.

We emphasize the fact that the presence of the term $\log(\Kp)$ complicates the analysis, as it incompatible with most linear algebra-based proof techniques, and it renders most information theory-based bounds too loose, as we explain in App~\ref{sec:concavity_details}. Nevertheless, we provide extensive numerical experiments in App.~\ref{sec:concavity_details} which support our conjecture.  In the remainder of this work, claims \textit{dependent} on this conjecture are labelled~$^\clubsuit$. 
 

\vspace{-2ex}

\subsection{Comparing probability distributions}

The previous conjecture implies that    $-\mathbb{H}^{\mathbf{K}}_1[\cdot]$ is a strictly convex function. This naturally suggests considering the Bregman divergence induced by the negative GAIT entropy. This is analogous to the construction of the Kullback-Leibler divergence as the Bregman divergence induced by the negative Shannon entropy.

Straightfoward computation shows that the gap between the negative GAIT entropy at $\vp$ and its linear approximation around $\vq$ evaluated at $\vp$ is:
\begin{align*}
    &-\mathbb{H}^{\mathbf{K}}_1[\vp] - \lspar - \mathbb{H}^{\mathbf{K}}_1[\vq]  + \left < -\nabla_{\vq} \mathbb{H}^{\mathbf{K}}_1[\vq], \, \vp - \vq \right > \rspar \\
    &= 1 + \left< \vp, \log  \frac{\Kp}{\Kq} \right > - \left <\vq, \frac{\Kp}{\Kq} \right> \stackrel{\text{(Conj. 1)}}{\ge} 0.
\end{align*}

\begin{definition}
	\label{def:divergence} \emph{(\textbf{GAIT Divergence})}$^\clubsuit$
	The GAIT divergence between distributions $\Prob$ and $\Qrob$ on a finite similarity space $\Xs$ is given by:
	\vspace{-1.5ex}
	\begin{align}
	    \label{eq:divergence}
        \hspace{-0mm}\BDiv [ \Prob \, || \, \Qrob ] \triangleq  1 + \Exp_{\Prob} \lspar  \log  \frac{\mK\Prob}{\mK\Qrob} \rspar - \Exp_{\Qrob} \lspar \frac{\mK\Prob}{\mK\Qrob} \rspar.
    \end{align}
\end{definition}

\begin{figure}[t]
    \centering
    \includegraphics[width=0.45\columnwidth]{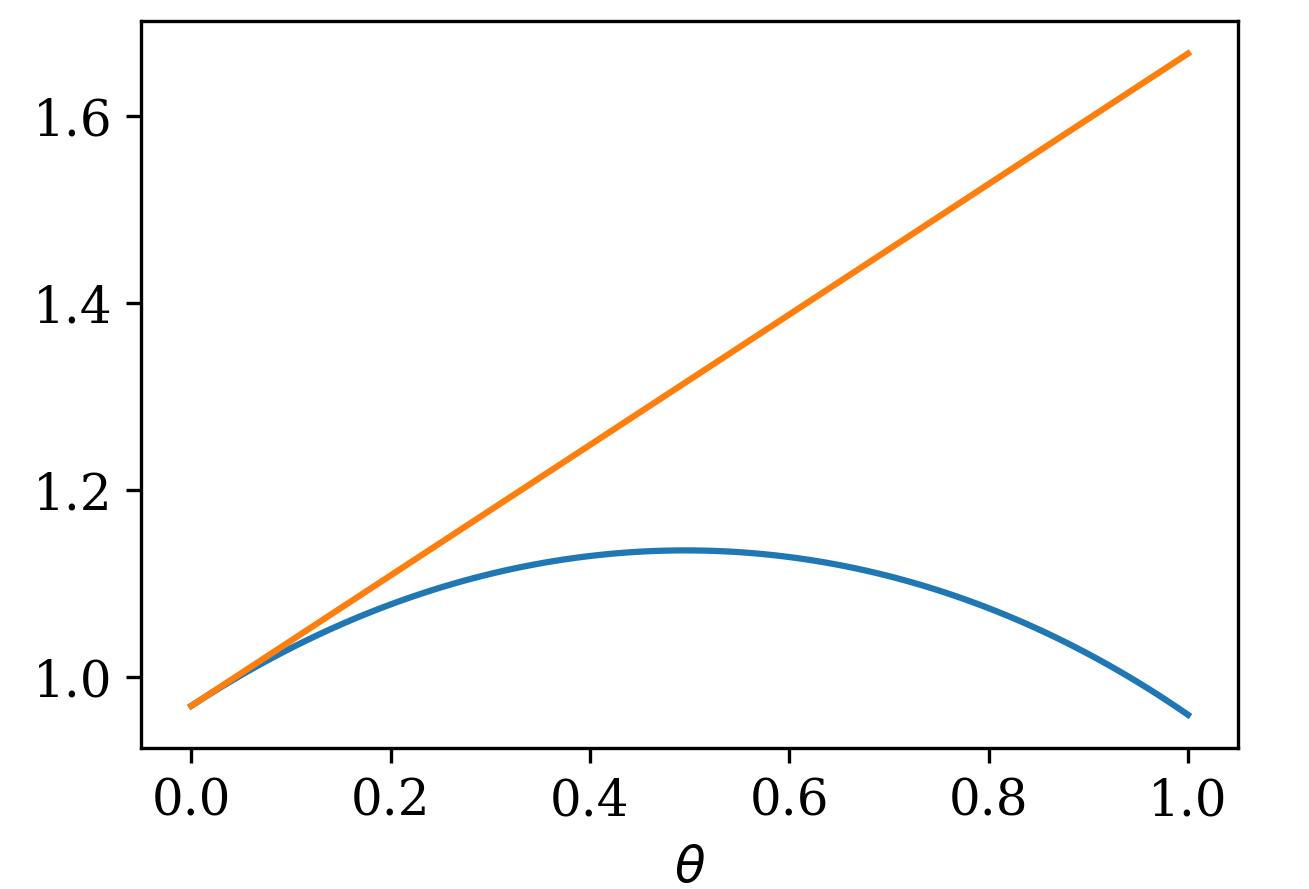}
    \includegraphics[width=0.5\columnwidth]{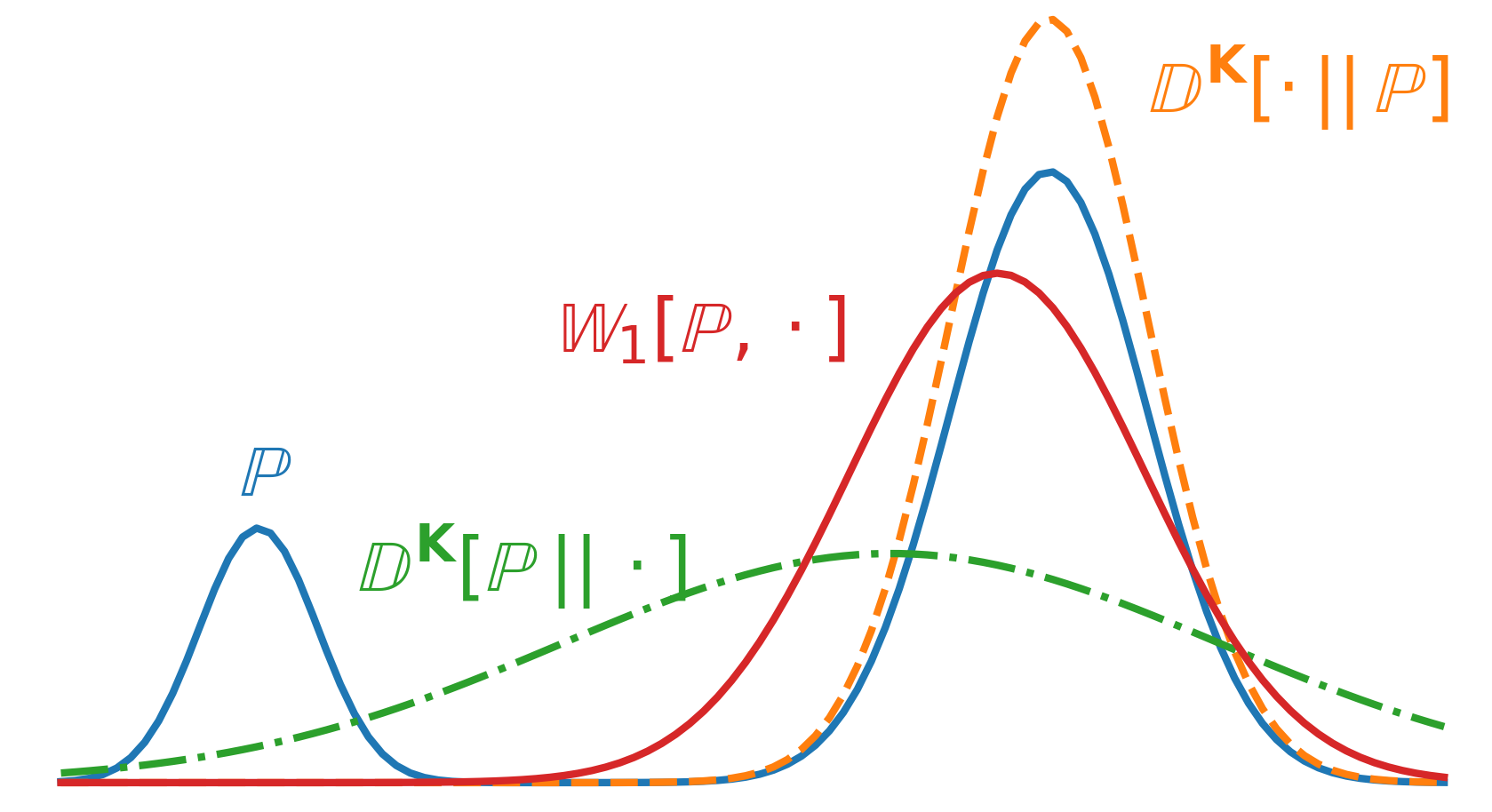}
    \caption{\textbf{Left:} The  {\color{blue}entropy} $\mathbb{H}^{\mathbf{K}}_1[(1 - \theta) \vq + \theta \vp]$ is upper-bounded by the  {\color{orange}linear approximation} at $\vq$, given by $\mathbb{H}^{\mathbf{K}}_1[\vq] + \theta  \left < \nabla_{\vq} \mathbb{H}^{\mathbf{K}}_1[\vq], \, \vp - \vq \right >$. \textbf{Right:} Optimal Gaussian model under various divergences on a simple mixture of Gaussians task under an RBF kernel. $\mathbb{W}_1$ denotes the 1-Wasserstein distance.}
    \label{fig:concave_and_mog}
\end{figure}

When $\mK = \eye$, the GAIT divergence reduces to the Kullback-Leibler divergence. Compared to the family of $f$-divergences ~\citep{f_div}, this definition computes point-wise ratios between the similarity profiles $\mK \Prob$ and $\mK \Qrob$ rather than the probability masses (or more generally, Radon-Nikodym w.r.t. a reference measure). We highlight that $\mK \Prob(x)$ provides a \emph{global} view of the space via the Gram matrix from the perspective of $x \in \X$.  Additionally, the GAIT divergence by definition inherits all the properties of Bregman divergences. In particular, $\BDiv[\Prob \, || \, \Qrob]$ is convex in $\Prob$.

\textbf{Forward and backward GAIT divergence.}
Like the Kullback-Leibler divergence, the GAIT divergence is not symmetric and different orderings of the arguments induce different behaviors. Let $\mathcal{Q}$ be a family of distributions in which we would like to find an approximation $\Qrob$ to $\Prob \notin \mathcal{Q}$. $\arg \min_\mathcal{Q} \BDiv[\cdot \, || \, \Prob]$ concentrates around one of the modes of $\Prob$; this behavior is known as \textit{mode seeking}. On the other hand, $\arg \min_\mathcal{Q} \BDiv[\Prob \, || \, \cdot]$ induces a \textit{mass covering} behavior. Fig. \ref{fig:concave_and_mog} displays this phenomenon when finding the best (single) Gaussian approximation to a mixture of Gaussians.

\begin{table*}[h!]
\centering
\caption{Definitions of GAIT mutual information and joint entropy.} \label{tab:definitions}
\begin{tabular}{|r|l|} \hline
     \textbf{Joint Entropy} &  $\mathbb{H}^{\mK \otimes \mLam}[X, Y] \triangleq - \Exp_{x, y \sim \Prob}[\log ([\mK \otimes \mLam] \Prob)_{x, y}]$ \\ \hline
     \textbf{Conditional Entropy} & $\mathbb{H}^{\mK, \mLam}[X | Y] \triangleq \mathbb{H}^{\mK \otimes \mLam}[X, Y] - \mathbb{H}^{\mLam}[Y]$ \\ \hline
     \textbf{Mutual Information} & $\mathbb{I}^{\mK, \mLam}[X ; Y] \triangleq \mathbb{H}^{\mK}[X] + \mathbb{H}^{\mLam}[Y] - \mathbb{H}^{\mK \otimes \mLam}[X, Y]$\\ \hline
     \textbf{Conditional M.I.} & $\mathbb{I}^{\mK, \mLam, \mThe}[X ; Y | Z] \triangleq \mathbb{H}^{\mK, \mThe}[X|Z] + \mathbb{H}^{\mLam, \mThe}[Y|Z] - \mathbb{H}^{\mK \otimes \mLam, \mThe}[X, Y | Z]$\\ \hline
\end{tabular}{}
\end{table*}

\textbf{Empirical distributions.} Although we have developed our divergence in the setting of distributions over a finite similarity space, we can effectively compare two empirical distributions over a continuous space. Note that if an arbitrary $x \in \X$ (or more generally a measurable set $E$ for a given choice of $\sigma$-algebra) has measure zero under both $\mu$ and $\nu$, then such $x$ (or $E$) is irrelevant in the computation of $\BDiv [ \Prob \, || \, \Qrob ]$. Therefore, when comparing empirical measures, the possibly continuous expectations involved in the extension of Eq. \eqref{def:divergence} to general measures reduce to finite sums over the corresponding supports.

Concretely, let $\Xs$ be a (possibly continuous) similarity space and consider the empirical distributions $\hat{\Prob} = \sum_{i = 1}^n \vp_i \delta_{x_i}$ and $\hat{\Qrob} = \sum_{j = 1}^m \vq_j \delta_{y_i}$ with $\vp \in \mathbf{\Delta}_n$ and $\vq \in \mathbf{\Delta}_m$. The Gram matrix of the restriction of $\Xs$ to $\mathcal{S} \triangleq \text{supp}(\Prob) \cup \text{supp}(\Qrob)$ has the block structure $
  \mK_{\mathcal{S}} \triangleq \begin{pmatrix} \mK_{xx} & \mK_{xy} \\  \mK_{yx} & \mK_{yy}  \end{pmatrix}$, where $\mK_{xx}$ is $n \times n$, $\mK_{yy}$ is $m \times m$ and $\mK_{xy} = \mK^\top_{yx}$.  It is easy to verify that
\begin{align}
    \BDiv [ \hat{\Prob} \, || \, \hat{\Qrob} ] = 1 + \left <\vp, \log  \frac{\mK_{xx}\vp}{\mK_{xy}\vq} \right>  -\left <\vq, \frac{\mK_{yx}\vp}{\mK_{yy}\vq}\right >.
    \label{eq:emp_loss}
\end{align}

\textbf{Computational complexity.} The computation of Eq. \eqref{eq:emp_loss} requires $\mathcal{O}(|\kappa|(n+m)^2)$ operations, where~$|\kappa|$ represents the cost of a kernel evaluation. This exhibits a quadratic behavior in the size of the union of the supports, typical of kernel-based approaches~\citep{mmd_gan}. We highlight that Eqs. \eqref{def:divergence} and \eqref{eq:emp_loss} provide a quantitative assessment of the dissimilarity between $\Prob$ and $\Qrob$ via a \emph{closed form expression}. This is in sharp contrast to the multiple variants of optimal transport which require the solution of an optimization problem or the execution of several iterations of matrix scaling algorithms. Moreover, the proposals of~\citet{cuturi_fast_bar, bregman} require at least $\Omega((|\kappa| + L)mn)$ operations, where $L$ denotes the number of Sinkhorn iterations, which is an increasing function of the desired optimization tolerance. A quantitative comparison is presented in App. \ref{sec:app_cplx}.

\textbf{Weak topology.} The type of topology induced by a divergence on the space of probability measures plays important role in the context of training neural generative models. Several studies~\citep{wgan, cuturi_learning, ot_gan} have exhibited how divergences which induce a weak topology constitute learning signals with useful gradients. In App.~\ref{sec:parallel}, we provide an example in which the GAIT divergence can provide a smooth training signal despite being evaluated on distribution with disjoint supports.

\subsection{Mutual Information}

We now use the GAIT entropy to define similarity-sensitive generalization of standard concepts related to mutual information.  As before, we restrict our attention to $\alpha=1$. This is required to get the chain rule of conditional probability for the R\'enyi entropy and to use Conj. \ref{conj:conc}. Finally, we note that although one could use the GAIT divergence to define a mutual information, in a fashion analogous to how traditional mutual information is defined via the KL divergence, the resulting object is challenging to study theoretically.  Instead, we use a definition based on entropy, which is equivalent in spaces without similarity structure.

\begin{definition} Let $X$, $Y$, $Z$ be random variables taking values on the similarity spaces $(\mathcal{X}, \kappa)$, $(\mathcal{Y}, \lambda)$, $(\mathcal{Z}, \theta)$ with corresponding Gram matrices $\mK$, $\mLam$, $\mThe$. Let $[\kappa \otimes \lambda] ((x, y), (x', y')) \triangleq \kappa(x, x') \lambda(y, y')$, and $(\mK \Qrob)_x \triangleq \Exp_{x' \sim \Qrob}[\kappa(x, x')]$ denotes the expected similarity between object $x$ and a random $\Qrob$-distributed object. Let $\Prob$ be the joint distribution of $X$ and $Y$. 
Then the joint entropy, conditional entropy, mutual information and conditional mutual information are defined following the formulas in Table. \ref{tab:definitions}.
\end{definition}
Note that the GAIT joint entropy is simply the entropy of the joint distribution with respect to the tensor product kernel.  This immediately implies monotonicity in the kernels $\mK$ and $\mLam$.  Note also that the chain rule of conditional probability holds by definition.

\begin{figure*}[t!]
    \vspace{-0.2cm}
    \centering
    \includegraphics[trim={1cm 0.5cm 1cm 1.8cm},clip,width=0.25\textwidth]{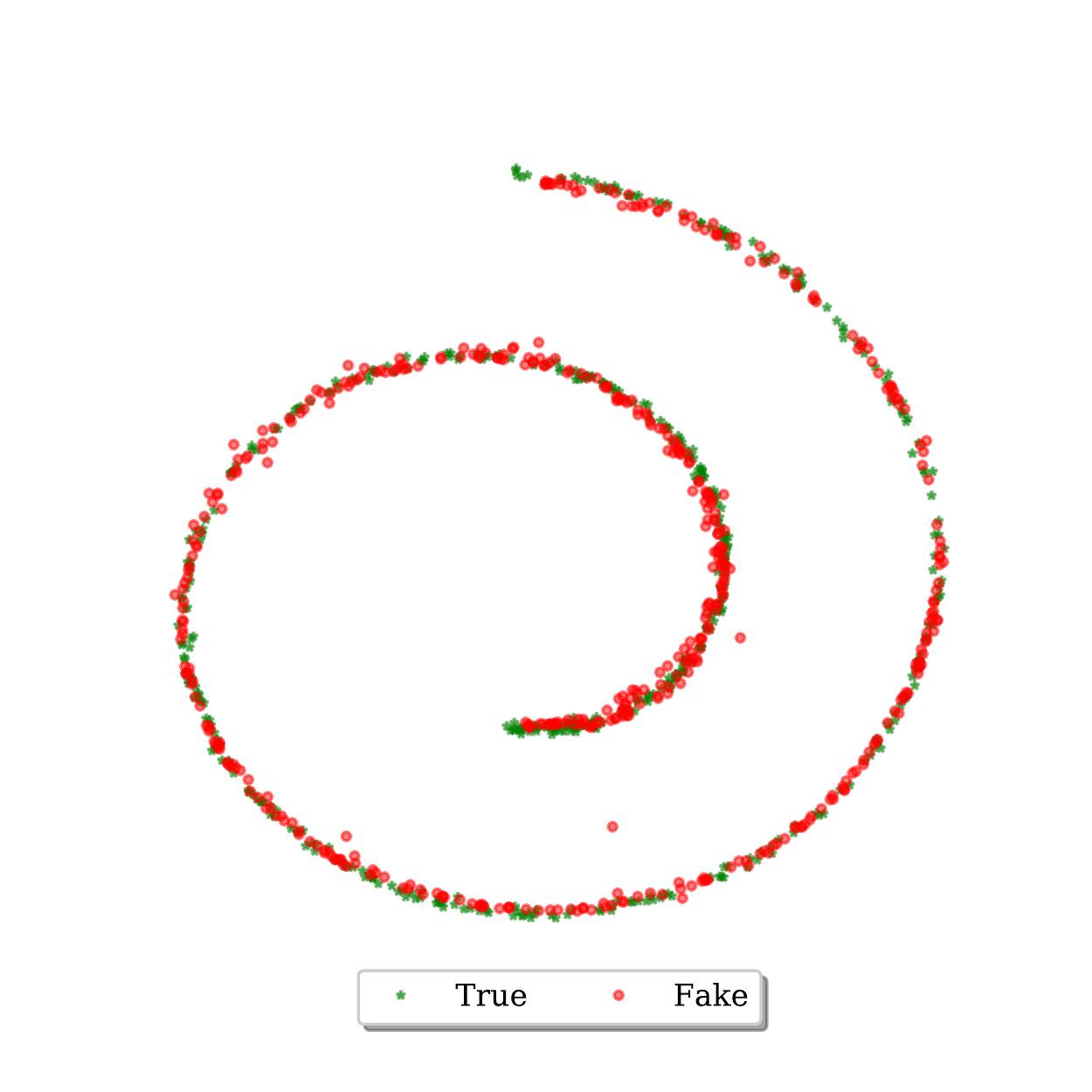}
    \hfill
    \includegraphics[width=0.29\textwidth]{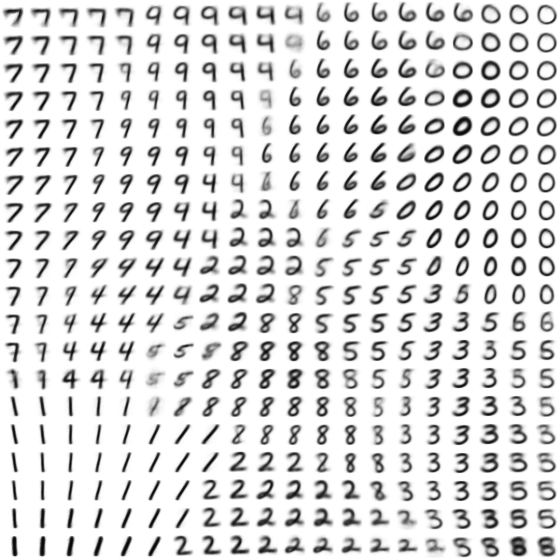} \hfill
    \includegraphics[width=0.29\textwidth]{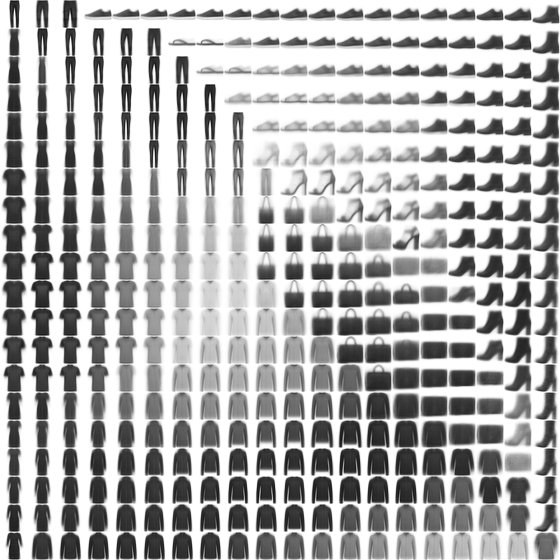}
    \hspace{0.7cm}
    \caption{\textbf{Left:} Generated Swiss roll data. \textbf{Center and Right:} Manifolds for MNIST and Fashion MNIST.}
    \label{fig:gen_models}
\end{figure*}

Subject to these definitions, similarity-sensitive versions of a number theorems analogous to standard results of information theory follow:
\begin{theorem}
	Let $X$, $Y$ be independent, then:
	\label{thm:indentropy}
	\begin{equation}
	\mathbb{H}^{\mK \otimes \mLam}[X, Y] = \mathbb{H}^{\mK}[X] + \mathbb{H}^{\mLam}[Y].
	\end{equation}
\end{theorem}
When the conditioning variables are perfectly identifiable ($\mLam = \eye$), we recover a simple expression for the conditional entropy:
\begin{theorem}
	For any kernel $\kappa$, \label{thm:conddecomp}
	\begin{equation}
	\hspace{0mm} \mathbb{H}^{\mK, \eye}[X | Y] = \Exp_{y \sim \Prob_y} [ \mathbb{H}^{\mK}[X | Y = y]].
	\end{equation}
\end{theorem}

Using Conj. \ref{conj:conc}, we are also able to prove that conditioning on additional information cannot increase entropy, as intuitively expected.
\begin{theorem}$^\clubsuit$
    \label{thm:ineq}
	For any similarity kernel $\kappa$,
	\begin{equation}
	\mathbb{H}^{\mK, \eye}[X | Y] \le \mathbb{H}^{\mK}[X].
	\end{equation}
\end{theorem}

Theorem \ref{thm:ineq} is equivalent to Conj. \ref{conj:conc} when considering a categorical $Y$ mixing over distributions $\{X_y\}_{y \in \Y}$.

Finally, a form of the data processing inequality (DPI), a fundamental result in information theory governing the mutual information of variables in a Markov chain structure, follows from Conj. \ref{conj:conc}.
\begin{theorem}  \emph{(\textbf{Data Processing Inequality})}$^\clubsuit$. \\
	If $X \rightarrow Y \rightarrow Z$ is a Markov chain, then \label{thm:dpi} 
	\begin{equation}
	\mathbb{I}^{\mK, \mThe}[X ; Z]  \le \mathbb{I}^{\mK, \mLam}[X ; Y] + 
	\mathbb{I}^{\mK, \mThe, \mLam}[X ; Z | Y].
	\end{equation}
\end{theorem}

Note the presence of the additional term $\mathbb{I}^{\mK, \mLam, \mThe}[X ; Z | Y]$ relative to the non-similarity-sensitive DPI given by $\mathbb{I}[X ; Z]  \le \mathbb{I}[X ; Y]$.  Intuitively, this  can be understood as reflecting that conditioning on $Y$ does not convey all of its usual ``benefit'', as some information is lost due to the imperfect identifiability of elements in $Y$. When $\mLam = \eye$ this term is 0, and the original DPI is recovered.

	\section{Related work}

\textbf{Theories of Information.} Information theory is ubiquitous in modern machine learning: from variable selection via information gain in decision trees \citep{uml}, to using entropy as a regularizer in reinforcement learning \citep{entropy_rl}, to rate-distortion theory for training generative models \citep{broken_elbo}. To the best of our knowledge, the work of \citet{leinster2012, leinster2016} is the first formal treatment of information-theoretic concepts in spaces with non-trivial geometry, albeit in the context of ecology.  

\textbf{Comparing distributions.} The ability to compare probability distributions is at the core of statistics and machine learning. Although traditionally dominated by maximum likelihood estimation, a significant portion of research on parameter estimation has shifted towards methods based on optimal transport, such as the Wasserstein distance \citep{Villani2008OptimalNew}. Two main reasons for this transition are (i) the need to deal with degenerate distributions (which might have density only over a low dimensional manifold) as is the case in the training of generative models \citep{gan, wgan, ot_gan}; and (ii) the development of alternative formulations and relaxations of the original optimal transport objective which make it feasible to approximately compute in practice \citep{cuturi_fast_bar, cuturi_learning}.


\textbf{Relation to kernel theory.} The theory we have presented in this paper revolves around a notion of similarity on $\X$. The operator $\mK \Prob$ corresponds to the embedding of the space of distributions on $\X$ into a reproducing kernel Hilbert space used for comparing distributions without the need for density estimation~\citep{hsed}. In particular, a key concept in this work is that of a characteristic kernel, i.e., a kernel for which the embedding is injective. Note that this condition is equivalent to the positive definiteness of the Gram matrix $\mK$ imposed above. Under these circumstances, the metric structure present in the Hilbert space can be imported to define the Maximum Mean Discrepancy distance between distributions \citep{ak2st}. Our definition of divergence also makes use of the object $\mK \Prob$, but has motivations rooted in information theory rather than functional analysis. We believe that the framework proposed in this paper has the potential to foster connections between both fields. 
	\section{Experiments}\label{sec:experiments}

\begin{figure*}[t!]
    \centering
    \begin{minipage}{.28\textwidth}
        \vspace{-0.8cm}
        \hspace{-2mm}
        \includegraphics[trim={2mm 8mm 5mm 6mm},clip,width=1\columnwidth]{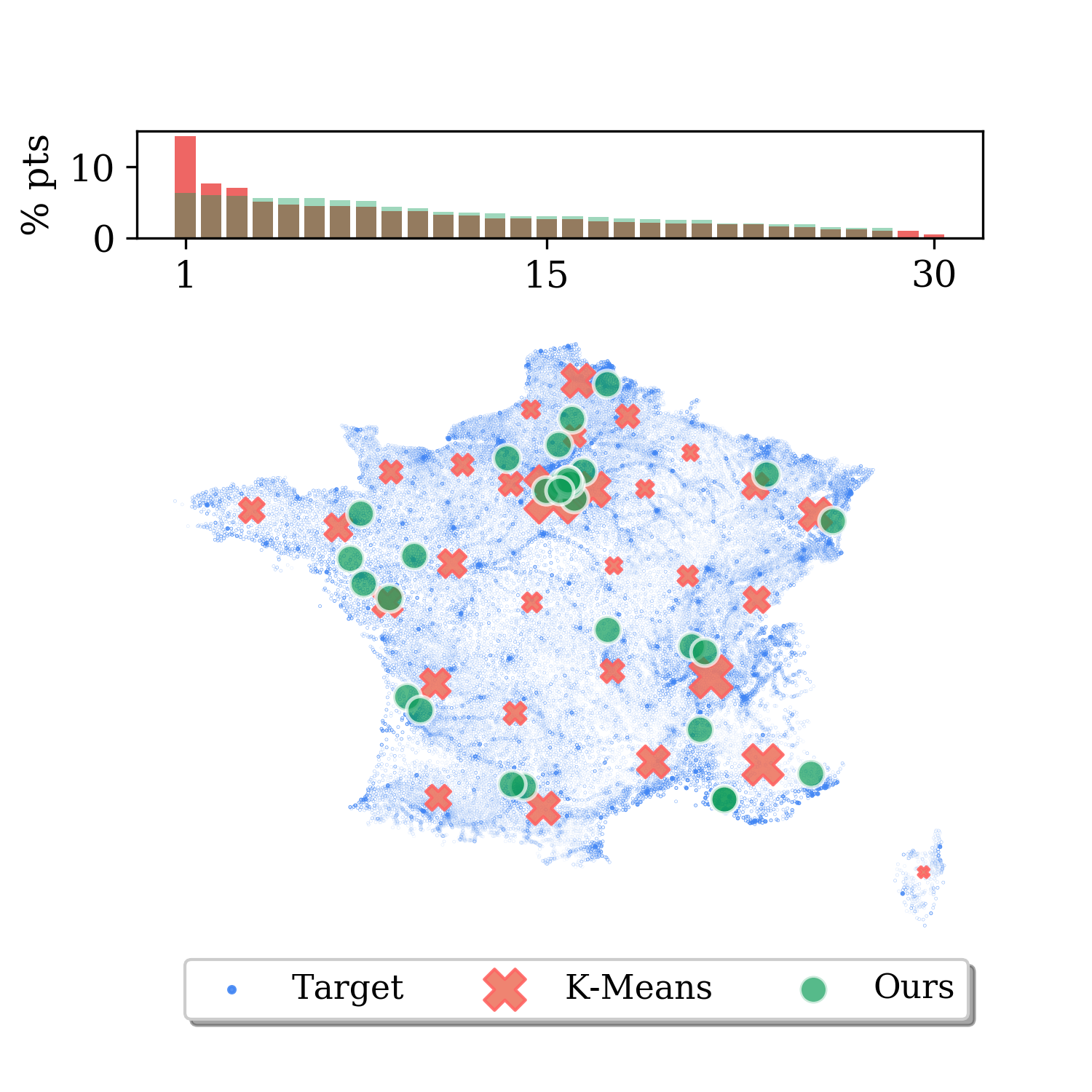} 
        \caption{Approximating a discrete measure with a uniform empirical measure.}
        \label{fig:atomic_approx}
    \end{minipage}
    \hfill
    \begin{minipage}{.29\textwidth}
        \vspace{-0.5cm}
        \includegraphics[trim={0mm 0mm 0mm 0mm},clip, width=0.98\columnwidth]{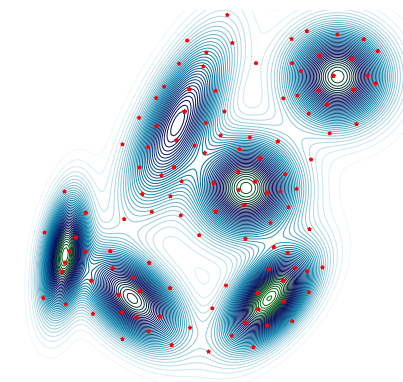}
        \vspace{-3mm}
        \caption{Approximating a continuous density with a finitely-supported measure.}
        \label{fig:supersamples}
    \end{minipage}\hfill
    \begin{minipage}{.3\textwidth}
        \centering
    	\includegraphics[height=0.37\textwidth, width=\columnwidth]{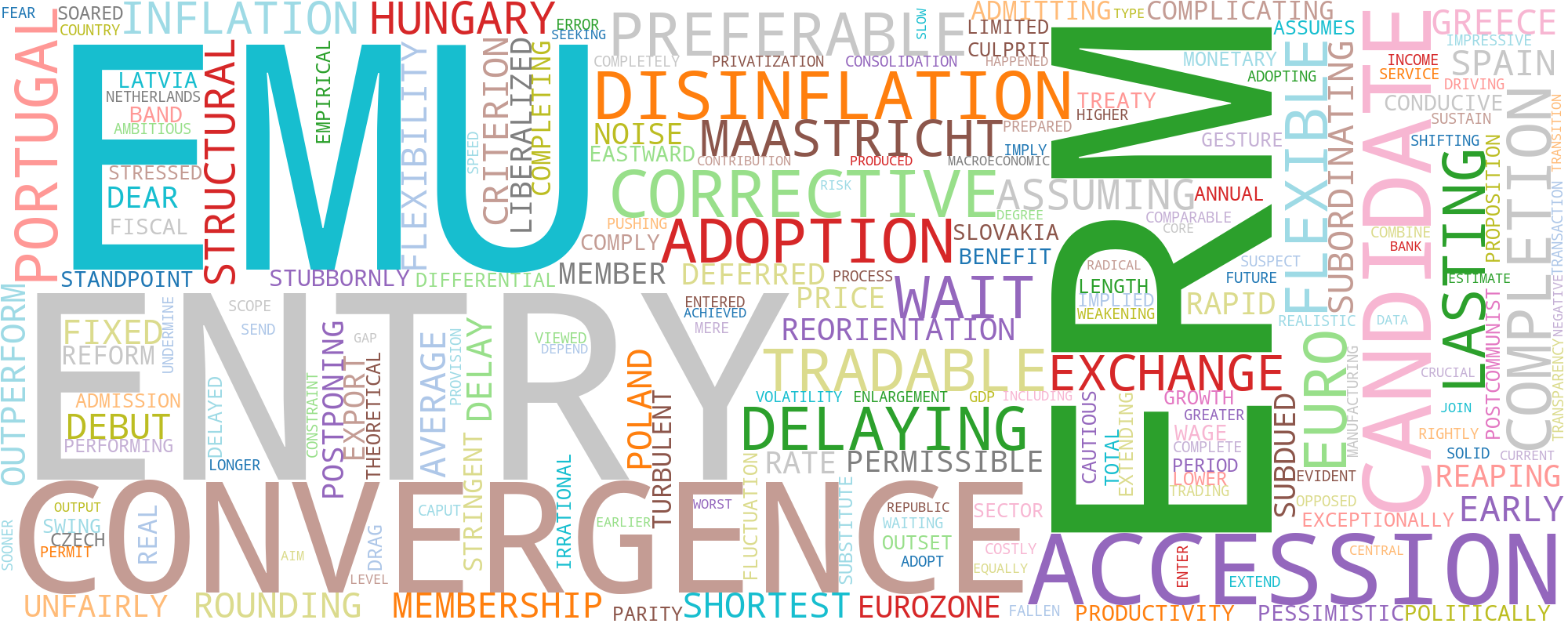} \\
    	\vspace{2mm}
    	\includegraphics[height=0.37\textwidth,width=0.48\columnwidth]{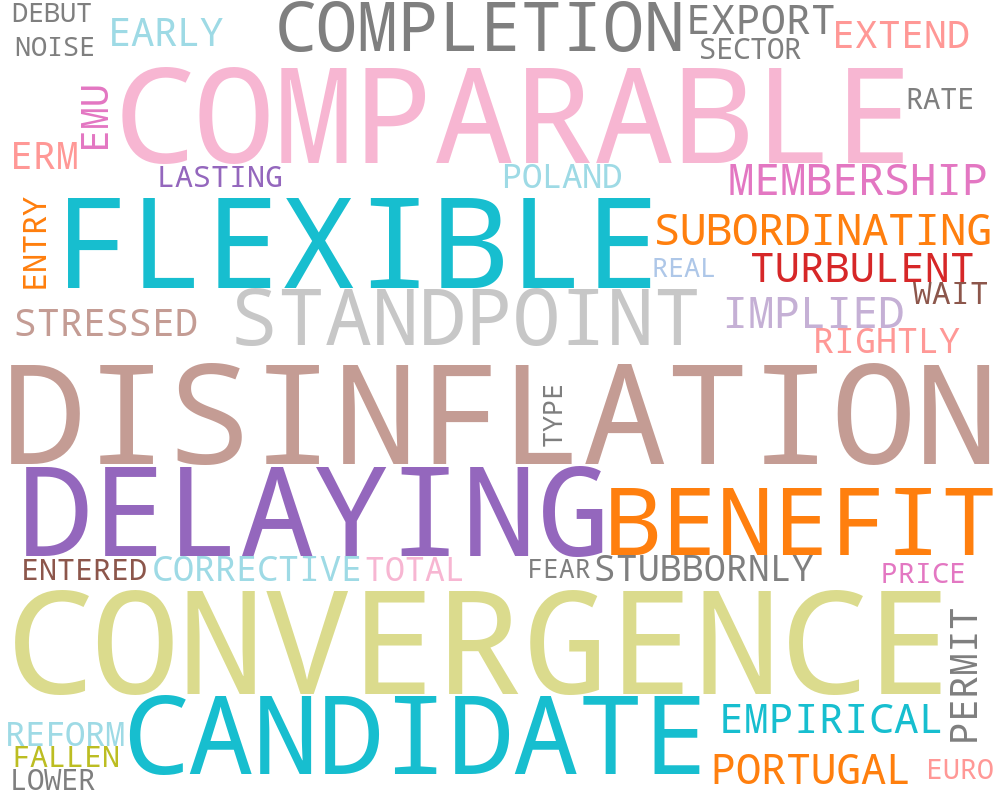} \hfill \includegraphics[height=0.37\textwidth,width=0.48\columnwidth]{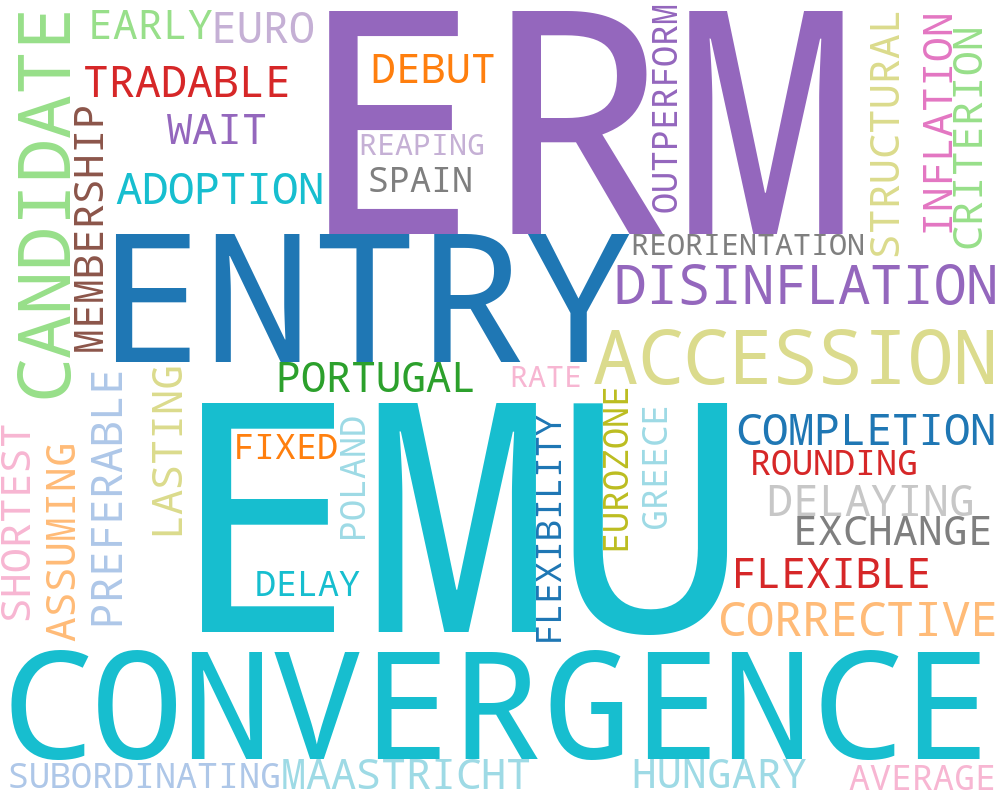}
    	\caption{\textbf{Top:} Original word cloud. \textbf{ Left:} Sparse approximation with support size 43. \textbf{Right:} Top 43 original TF-IDF words.}
    	\label{fig:wordclouds}
    \end{minipage}
\end{figure*}

    

\subsection{Comparison to Optimal Transport} \label{sec:comparison_ot}

\textbf{Image barycenters.} Given a collection of measures $\mathcal{P} = \{\Prob_i\}^n_{i=1}$ on a similarity space, we define the barycenter of $\mathcal{P}$ with respect to the GAIT divergence as $ \arg \min_{\Qrob} \frac{1}{n} \sum_{i=1}^n \BDiv[\Prob_i \, || \, \Qrob]$.  This is inspired by the work of \citet{cuturi_fast_bar} on Wasserstein barycenters. Let the space $\X = [1:28]^2$ denote the pixel grid of an image of size $28 \times 28$. We consider each image in the MNIST dataset as an empirical measure over this grid in which the probability of location $(x, y)$ is proportional to the intensity at the corresponding pixel. In other words, image $i$ is considered as a measure $\Prob_i \in \mathbf{\Delta}_{|\X|}$. Note that in this case the kernel is a function of the distance between two pixels in the grid (two elements of $\X$), rather than the distance between two different images. We use a Gaussian kernel, and compute $\mK \Prob_i$ by convolving the image $\Prob_i$ with an adequate filter, as proposed by \citet{conv_ot}.

\begin{figure}[h]
    \centering
    \includegraphics[trim={0.5cm 0.0cm 0.0cm 0.2cm},clip, width=1.0\columnwidth]{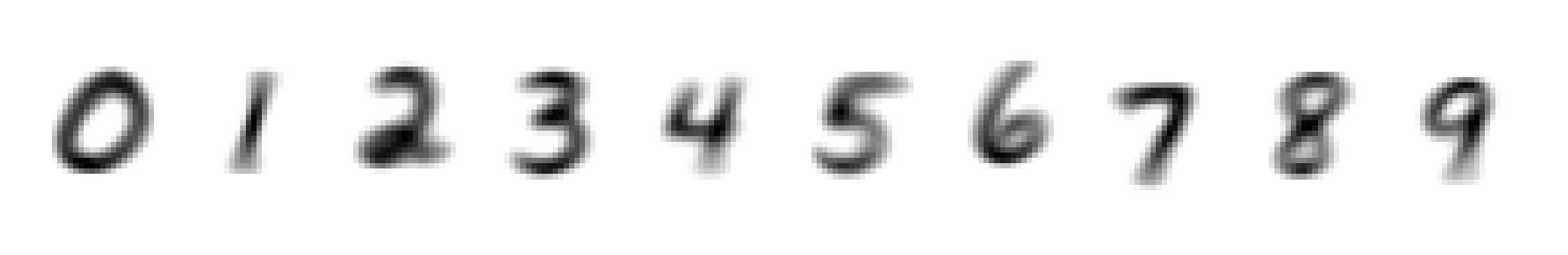}\\
    \includegraphics[trim={0.5cm 0.5cm 0.1cm 0.5cm},clip,width=0.97\columnwidth]{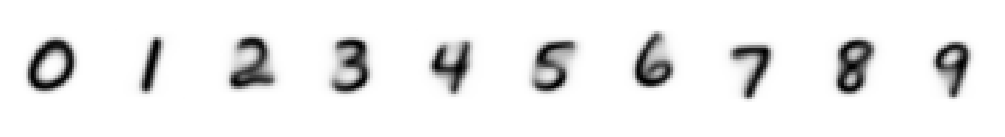}
    \caption{Barycenters for each class of MNIST with our divergence (top) and the method of \citet{cuturi_fast_bar} (bottom).}
    \label{fig:mnist_barys}
\end{figure}

Fig. \ref{fig:mnist_barys} shows the result of gradient-based optimization to find barycenters for each of the classes in MNIST~\citep{mnist} along with the corresponding results using the method of \citet{cuturi_fast_bar}. We note that our method achieves results of comparable quality. Remarkably, the time for computing the barycenter for each class on a single CPU is reduced from 90 seconds using the efficient method proposed by~\citet{cuturi_fast_bar, bregman} (implemented using a convolutional kernel \citep{conv_ot}) to less than 5 seconds using our divergence. Further experiments can be found in App. \ref{sec:interp}.

\textbf{Generative models.} The GAIT divergence can also be used as an objective for training generative models. We illustrate the results of using our divergence with a RBF kernel to learn generative models in Fig.~\ref{fig:gen_models} on a toy Swiss roll dataset, in addition to the MNIST \citep{mnist} and Fashion-MNIST  \citep{fashion_mnist} datasets. For all three datasets, we consider a 2D latent space and replicate the experimental setup used by \citet{cuturi_learning} for MNIST. We were able to use the same $2$-layer multilayer perceptron architecture and optimization hyperparameters for all three datasets, requiring only the tuning of the kernel variance for Swiss roll data's scale.

Moreover, we do not need large batch sizes to get good quality generations from our models. The quality of our samples obtained using batch sizes as small as $50$ are comparable to the ones requiring batch size of $200$ by~\citet{cuturi_learning}. We include additional experimental details and results in App.~\ref{sec:app_genmodels}, along with comparisons to variational auto-encoders \citep{vae}.

\subsection{Approximating measures}


\begin{figure*}[ht!]
    \vspace{-2mm}
    \centering
    \includegraphics[trim={0cm 0.0cm 0.cm 0.cm},clip,width=.32\textwidth]{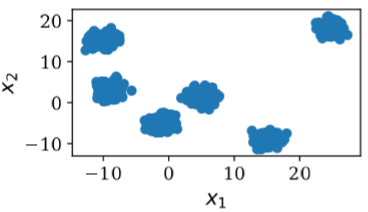}
    \includegraphics[trim={0cm 0.0cm 0cm .0cm},clip,width=0.32\textwidth]{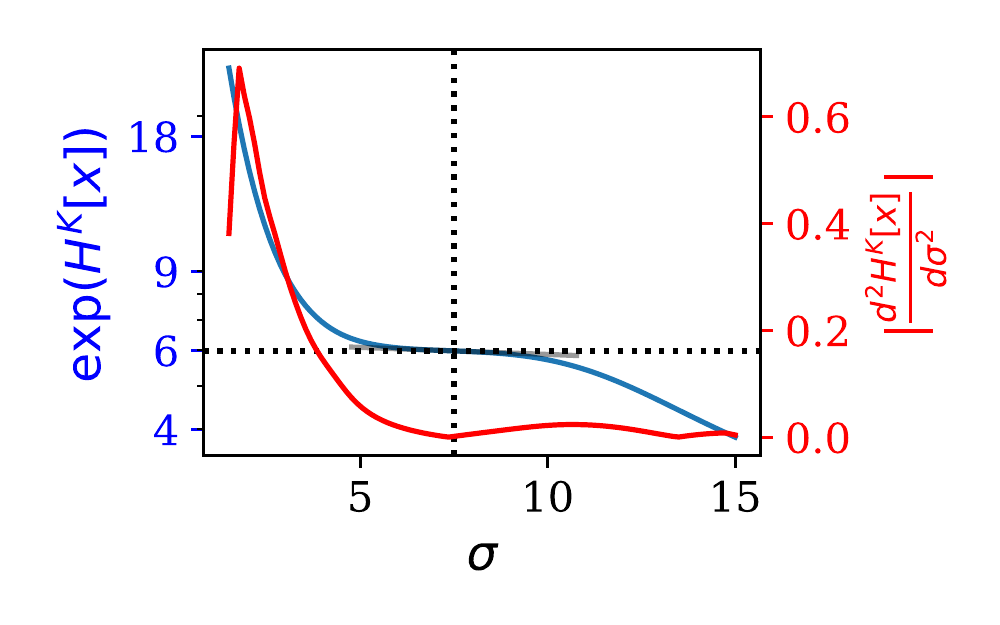}
    \includegraphics[trim={0cm 0.0cm 0cm .0cm},clip,width=0.32\textwidth]{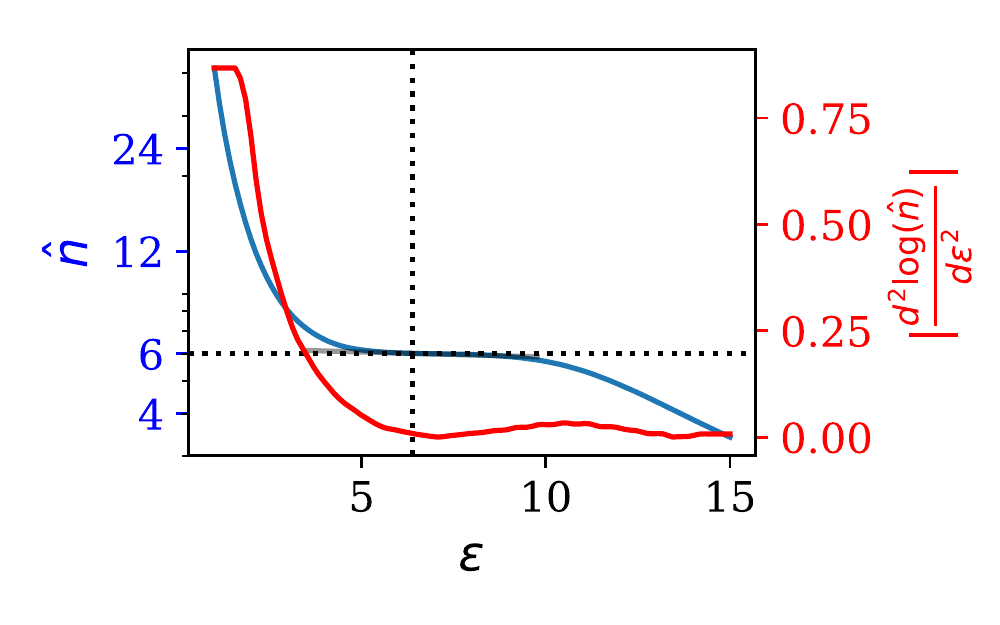}
    
    \vspace{-3mm}
    \caption{\textbf{Left:} 1,000 samples from a mixture of 6 Gaussians.  \textbf{Center:} Modes detected by varying $\sigma$ in our method. \textbf{Right:} Modes detected by varying collision threshold $\epsilon$ in the birthday paradox-based method.}
    
    \label{fig:mog}
\end{figure*}

Our method allows us to find a finitely-supported approximation $\Qrob = \sum_{j = 1}^m \vq_j \delta_{y_i}$ to a (discrete or continuous) target distribution $\Prob$. This is achieved by minimizing the divergence $\BDiv[\Prob || \Qrob]$ between them with respect to the locations $\{y_i\}_{i=1}^m$ and/or the masses of the atoms $\vq \in \mathbf{\Delta}_m$ in the approximating measure. In this section, we consider situations where $\Prob$ is not a subset of the support of $\Qrob$.  As a result, the Kullback-Leibler divergence (the case $\vect{K} = \eye)$ would be infinite and could not be minimized via gradent-based methods.  However, the GAIT divergence can be minimized even in the case of non-overlapping supports since it takes into account similarities between items.

In Fig. \ref{fig:atomic_approx}, we show the results of such an approximation on data for the population of France in 2010 consisting of 36,318 datapoints \citep{france_data}, similar to the setting of \citet{cuturi_fast_bar}.  The weight of each atom in the blue measure is proportional to the population it represents. We use an RBF kernel and an approximating measure consisting of 50 points with uniform weights, and use gradient-based optimization to minimize $\BDiv$ with respect to the location of the atoms of the approximating measure.  We compare with K-means \citep{sklearn} using identical initialization. Note that when using K-means, the resulting allocation of mass from points in the target measure to the nearest centroid can result in a highly unbalanced distribution, shown in the bar plot in orange. In contrast, our objective allows a uniformity constraint on the weight of the centroids, inducing a more homogeneous allocation. This is important in applications where an imbalanced allocation is undesirable, such as the placement of hospitals or schools.

Fig. \ref{fig:supersamples} shows the approximation of the density of a mixture of Gaussians $\Prob$ by a uniform distribution $\Qrob = \frac{1}{N} \sum_{i=1}^N \delta_{x_i}$ over $N=200$ atoms with a polynomial kernel of degree 1.5, similar to the approximate super-samples \citep{ssamples} task presented by \citet{stoch_w_barys} using the Wasserstein distance. We minimize $\BDiv[\Prob \, || \, \Qrob]$ with respect to the locations $\{x_i\}_{i=1}^n$. We estimate the continuous expectations with respect to $\Prob$ by repeatedly sampling minibatches to construct an empirical measure $\hat{\Prob}$. Note how the solution is a ``uniformly spaced'' allocation of the atoms through the space, with the number of points in a given region being proportional to mass of the region. See App. \ref{sec:interp} for a comparison to \citet{stoch_w_barys}.

Finally, one can approximate a measure when the locations of the atoms are fixed.  As an example, we take an article from the News Commentary Parallel Corpus \citep{newscorpus}, using as a measure $\Prob$ the normalized TF-IDF weights of each non-stopword in the article. Here, $\vect{K}$ is given by an RBF kernel applied to the $300$-dimensional GLoVe \citep{glove} embeddings of each word.  We optimize $\Qrob$ applying a penalty to encourage sparsity.  We show the result of this summarization in word-cloud format in Fig. \ref{fig:wordclouds}.  Note that compared to TF-IDF, which places most mass on a few unusual words, our method produces a summary that is more representative of the original text.  This behavior can be modified by varying the bandwidth $\sigma$ of the kernel, producing approximately the same result as TF-IDF when $\sigma$ is very small; details are presented in App. \ref{sec:text}.

\vspace{-1ex}

\subsection{Measuring diversity and counting modes}

As mentioned earlier, the exponential of the entropy $\exp(\mathbb{H}^{\mathbf{K}}_1[\Prob])$ provides a measure of the effective number of points in the space 
\citep{leinster_magnitude}. 
In Fig. \ref{fig:mog}, 
we use an empirical distribution to estimate the number of modes of a mixture of $C$ Gaussians. As the kernel bandwidth $\sigma$ increases, $\exp(\mathbb{H}^{\mathbf{K}}_1[\hat{\Prob}])$ decreases, with a marked plateau around $C$. We highlight that the lack of direct consideration of geometry of the space in the Shannon entropy renders it useless here: at any (non-trivial) scale, $\exp(\mathbb{H}[\hat{\Prob}])$ equals the number of samples, and not the number of classes. Our approach obtains similar results as (a form of) the birthday paradox-based method of \citet{birthday}, while avoiding the need for human evaluation of possible duplicates. Details and tests on MNIST can be found in App. \ref{sec:mode_details}.
	\section{Conclusions}

In this paper, we advocate the use of geometry-aware information theory concepts in machine learning.  We present the similarity-sensitive entropy of~\citet{leinster2012} along with several important properties that connect it to fundamental notions in geometry. 
We then propose a divergence induced by this entropy, which compares probability distributions by taking into account the similarities among the objects on which they are defined. Our proposal shares the empirical performance properties of distances based on optimal transport theory, such as the Wasserstein distance~\citep{Villani2008OptimalNew}, but enjoys a closed-form expression. This obviates the need to solve a linear program or use matrix scaling algorithms~\citep{cuturi_regularized}, reducing computation significantly. Finally, we also propose a similarity-sensitive version of mutual information based on the GAIT entropy. We hope these methods can prove fruitful in extending frameworks such as the information bottleneck for representation learning \citep{info_bottle}, similarity-sensitive cross entropy objectives 
in the spirit of loss-calibrated decision theory~\citep{loss_cal_lacoste}, or the use of entropic regularization of policies in reinforcement learning~\citep{entropy_rl}.




	
	\subsubsection*{Acknowledgments}
	
	This research was partially supported by the Canada CIFAR AI Chair Program and by a Google Focused Research award. Simon Lacoste-Julien is a CIFAR Associate Fellow in the Learning in Machines \& Brains program. We thank Pablo Piantanida for the great tutorial on information theory which inspired this work, and Mark Meckes for remarks on terminology and properties of metrics spaces of negative type.

	
	\bibliographystyle{abbrvnatClean} 
	\bibliography{curated_refs.bib}
	
	\newpage
	
	\appendix
	
	\onecolumn

\section{Revisiting parallel lines}
\label{sec:parallel}

Let $Z \sim \Uniform([0, 1])$ , $\phi \in \reals $, and let $\Prob_{\phi}$ be the distribution of $(\phi, Z) \in \reals^2$, i.e., a (degenerate) uniform distribution on the segment $\{\phi \} \times [0, 1] \subset \reals^2$, illustrated in Fig. \ref{fig:par_lines}.

\begin{figure}[h]
    \centering
	\begin{minipage}{.45\textwidth}
		\centering
		\begin{tikzpicture}
		\draw[->] (0,0)--(3, 0) node[right]{};
		\draw[->] (0,0)--(0, 2.2) node[above]{};
		\draw[-, blue, ultra thick] (1,0)--(1,1) node[right]{$\Prob_1$};
		\draw[-, red, ultra thick] (0,0)--(0,1) node[right]{$\Prob_0$};
		\draw[-, purple, ultra thick] (2,0)--(2,1) node[right]{$\Prob_2$};
		\node at (2.5, 2) {$\reals^2$};
		\end{tikzpicture}
		\captionof{figure}{Distribution $\Prob_{\phi}$ with support on the 1-dim segment $\{\phi \} \times [0, 1]$ for different values of $\phi$.}
		\label{fig:par_lines}
	\end{minipage}
	\hspace{1cm}
	\begin{minipage}{.45\textwidth}
		\vspace{-1mm}
		\centering
		\begin{tikzpicture}[scale=0.7]
			\begin{axis}[
			axis lines = middle,
			xmin=-1, xmax=1,
			ymin=0, ymax=1.2,
			legend style={at={(1.15, 0.35)},anchor=west}, 
			legend columns=1,
			width=7.cm, height=5cm,
			xlabel = $\phi$,
			x label style={at={(axis description cs:1,0.2)},anchor=north},
			]
		
			\addplot [domain=-1:1, samples=100, color=red, thick]{1};
			\addlegendentry{$\delta(\Prob_0, \Prob_{\phi})$}
			
			\addplot [domain=-1:1, samples=100, color=blue, thick]{0.69};
			\addlegendentry{$\mathbb{JS}(\Prob_0, \Prob_{\phi})$}
			
			
			\addplot [domain=-1:1, samples=100, color=green, thick]{abs(x)};
			\addlegendentry{$\mathbb{W}_1(\Prob_0, \Prob_{\phi})$}
			
			\addplot [domain=-1:1, samples=100, color=purple, thick]{x^2 + 1 - exp(-x^2)};
			\addlegendentry{$\BDiv(\Prob_{\phi}, \Prob_0) $}
			
			\end{axis}
			
			\node [red] at (2.72, 2.8) {$\circ$};
			\node [blue] at (2.72, 1.92) {$\circ$};
			\node [black] at (2.72, 0) {$\bullet$};
		
		\end{tikzpicture}
		\captionof{figure}{Values of the divergences as functions of $\phi$. KL divergence values are $\infty$ except at $\phi = 0$.}
		\label{fig:divtheta}
	\end{minipage}
\end{figure}

Our goal is to find the \textit{right} value of $\phi$ for a model distribution $\Prob_{\phi}$ using the dissimilarity with respect to a target distribution $\Prob_{0}$ as a learning signal. The behavior of common divergences on this type of problem was presented by \cite{wgan} as a motivating example for the introduction of OT distances in the context of GANs.
\begin{equation*}
    \begin{split}
        \delta(\Prob_0, \Prob_\phi) &= \begin{cases} 0   \hspace{5mm} \text{if } \phi = 0\\ 1  \hspace{5mm} \text{else}\end{cases}
        \hspace{7mm}
        \mathbb{KL}(\Prob_0, \Prob_\phi) = \mathbb{KL}(\Prob_\phi, \Prob_0) = \begin{cases} 0   \hspace{5mm} \text{if } \phi = 0\\ \infty  \hspace{5mm} \text{else}\end{cases}
        \\
        \mathbb{W}_1(\Prob_0, \Prob_\phi) &= |\phi|
        \hspace{28.5mm}
        \mathbb{JS}(\Prob_0, \Prob_\phi) = \log(2) \, \delta(\Prob_0, \Prob_\phi)
    \end{split}
\end{equation*}

Note that among all these divergences, illustrated in Fig. \ref{fig:divtheta}, only the Wasserstein distance provides a continuous (even a.e. differentiable) objective on $\phi$. We will now study the behavior of the GAIT divergenve in this setting.

Recall that the action of the kernel on a given probability measure corresponds to the mean map $\mK \mu : \X \to \reals$, defined by $\mK \mu(x) \triangleq \Exp_{x' \sim \mu} \lspar \kappa(x, x') \rspar = \int \kappa(x, x') \diff \mu(x')$. In particular, for $\Prob_\phi$:
\begin{equation*}
    \mK \Prob_\phi (x, y) = \int_{\reals^2} \kappa( (x, y), (x', y')) \diff \Prob_\phi(x', y') =  \int_0^1 \kappa( (x, y), (\phi, y')) \diff y'.
\end{equation*}

Let us endow $\reals^2$ with the Euclidean norm $\norm{\cdot}_2$, and define the kernel $\kappa((x, y), (x', y')) \triangleq \exp(-\norm{(x, y) - (x', y')}_2^2)$. Note that this choice is made only for its mathematically convenience in the following algebraic manipulation, but other choices of kernel are possible.  In this case, the mean map reduces to:
\begin{equation*}
    \mK \Prob_\phi (x, y) = \int_0^1 \exp \lspar - \lpar (x - \phi)^2 + (y - y')^2 \rpar \rspar \diff y' = \exp \lspar -(x - \phi)^2 \rspar     \underbrace{\int_0^1 \exp \lspar - (y - y')^2 \rspar \diff y'.}_{\text{$\triangleq I_y$, independent of $\phi$.}}
\end{equation*}

\vspace{-2ex}

We obtain the following expressions for the terms appearing in the divergence:
\begin{equation*}
    \Exp_{(x, y) \sim \Prob_\phi} \log \lspar \frac{\mK \Prob_\phi (x, y)}{\mK \Prob_0 (x, y)} \rspar = \Exp_{(x, y) \sim \Prob_\phi} \log \lspar \frac{\exp \lspar -(x - \phi)^2 \rspar  \bcancel{I_y}}{\exp \lspar -x^2 \rspar \bcancel{I_y}} \rspar = \Exp_{(x, y) \sim \Prob_\phi} \log \exp \lspar x^2 -(x - \phi)^2 \rspar = \phi^2.
\end{equation*}

\begin{equation*}
    \Exp_{(x, y) \sim \Prob_0} \lspar \frac{\mK \Prob_\phi (x, y)}{\mK \Prob_0 (x, y)} \rspar = \Exp_{(x, y) \sim \Prob_0} \exp \lspar x^2 -(x - \phi)^2 \rspar = \exp \{-\phi^2 \}.
\end{equation*}

Finally, we replace the previous expressions in the definition of the GAIT divergence. Remarkably, the result is a smooth function of the parameter $\phi$ with a global optimum at $\phi = 0$. See Fig. \ref{fig:divtheta}.

\begin{equation*}
    \BDiv(\Prob_\phi, \Prob_0) = 1 + \Exp_{(x, y) \sim \Prob_\phi} \log \lspar \frac{\mK \Prob_\phi(x, y)}{\mK \Prob_0(x, y)} \rspar - \Exp_{(x, y) \sim \Prob_0} \lspar \frac{\mK \Prob_\phi(x, y)}{\mK \Prob_0(x, y)} \rspar = \phi^2 + 1 - e^{-\phi^2} \ge 0. 
\end{equation*}

	\section{Proofs} \label{sec:proofs}

\setcounter{theorem}{2}

\begin{theorem}
	Let $X$, $Y$ be independent, then $ 	\mathbb{H}^{\mK \otimes \mLam}[X, Y] = \mathbb{H}^{\mK}[X] + \mathbb{H}^{\mLam}[Y].$
\end{theorem}
\begin{proof}
\begingroup
\addtolength{\jot}{0.5em}
\begin{align*}
\vspace{-1cm} \mathbb{H}^{\mK \otimes \mat{J}}[X, Y] &=\Exp_{x, y} \log \lspar \Exp_{x', y'} \kappa(x, x')\lambda(y, y') \rspar \\ &=\Exp_{x, y} \log \lspar \Exp_{x'}\lspar\Exp_{y'} \kappa(x, x')\lambda(y, y') \rspar \rspar\\ 
&=\Exp_{x, y} \log \lspar \Exp_{x'}\lspar\kappa(x, x')\rspar \Exp_{y'} \lspar \lambda(y, y') \rspar \rspar\\
&=\Exp_{x, y} \log \lspar \Exp_{x'}\lspar\kappa(x, x')\rspar \rspar + \log \lspar \Exp_{y'} \lspar \lambda(y, y') \rspar \rspar\\
&=\Exp_{x} \log \lspar \Exp_{x'}\lspar\kappa(x, x')\rspar \rspar + \Exp_{y} \log \lspar \Exp_{y'} \lspar \lambda(y, y') \rspar \rspar\\
&= \mathbb{H}^{\mK}[X] + \mathbb{H}^{\mLam}[Y]. \vspace{-0.2cm}
\end{align*}
\endgroup
\end{proof}

\vspace{-2ex}

\begin{theorem}
	For any kernel $\kappa$, $\mathbb{H}^{\mK, \eye}[X | Y] = \Exp_{y \sim \Prob_y} [ \mathbb{H}^{\mK}[X | Y = y]]$.
\end{theorem}
\begin{proof}
\begingroup
\addtolength{\jot}{0.5em}
\begin{align*}
    \vspace{-1cm} \mathbb{H}^{\mK, \eye}[X | Y] &= \mathbb{H}^{\mK, \eye}[X, Y] - \mathbb{H}^{\eye}[Y]\\ 
    &= \Exp_{x, y} \log \lspar \Exp_{x', y'} \kappa(x, x')\mathbf{1}(y, y') \rspar - \mathbb{H}^{\eye}[Y]   \\
    &= \Exp_{x, y} \log \lspar \int_{x'}\int_{y'}p(x', y') \kappa(x, x')\mathbf{1}(y, y') \rspar - \mathbb{H}^{\eye}[Y]  \\
    &= \Exp_{x, y} \log \lspar \int_{x'}p(x' |y)p(y) \kappa(x, x')\rspar - \mathbb{H}^{\eye}[Y]   \\  
    &= \Exp_{x, y} \log \lspar p(y)\Exp_{x'|y}\kappa(x, x')\rspar - \mathbb{H}^{\eye}[Y]   \\
    &= \Exp_{x, y} \log \lspar \Exp_{x'|y}\kappa(x, x')\rspar   \\
    &= \Exp_{y} \lspar \Exp_{x|y} \log \lspar \Exp_{x'|y}\kappa(x, x')\rspar \rspar  \\
    &= \Exp_{y} \lspar \mathbb{H}^{\mK}[X | y] \rspar. \vspace{-0.2cm}
\end{align*}{}
\endgroup
\end{proof}{}

\begin{theorem}$^\clubsuit$
	For any similarity kernel $\kappa$, $\mathbb{H}^{\mK, \eye}[X | Y] \le \mathbb{H}^{\mK}[X].$
\end{theorem}
\begin{proof} $
    \mathbb{H}^{\mK, \eye}[X | Y] = \Exp_{y \sim \Prob_y} [ \mathbb{H}^{\mK}[X | Y = y]] = \Exp_{y \sim \Prob_y} [ \mathbb{H}^{\mK}[X | Y = y]] \stackrel{\text{(Jensen)}}{\le} \mathbb{H}^{\mK}[\Exp_{y \sim \Prob_y} [X | Y = y]] = \mathbb{H}^{\mK}[X].
$
\end{proof}{}

\begin{lemma} \label{thm:mi_chain} \emph{(\textbf{Chain Rule of Mutual Information})}$^\clubsuit$. $\mathbb{I}^{\mK, \mLam, \mThe}[X ; Y, Z] = \mathbb{I}^{\mK, \mLam}[X ; Y] + \mathbb{I}^{\mK, \mLam, \mThe}[X ; Y | Z]$
\end{lemma}

\begin{proof}
By definition:
\begingroup
\addtolength{\jot}{0.5em}
\begin{align*}
    \mathbb{I}^{\mK, \mThe}[X ; Z] &= \mathbb{H}^\mK [X] + \mathbb{H}^\mThe [Z] - \mathbb{H}^{\mK \otimes \mThe}[X, Z]. \\
    \mathbb{I}^{\mK, \mLam, \mThe}[X ; Y | Z] &= \mathbb{H}^{\mK, \mThe} [X|Z] + \mathbb{H}^{\mLam, \mThe} [Y|Z] - \mathbb{H}^{\mK, \mLam, \mThe} [X, Y|Z] \\
    &= \mathbb{H}^{\mK, \mThe} [X, Z] - \mathbb{H}^{\mThe} [Z] + \mathbb{H}^{\mLam, \mThe} [Y, Z] - \mathbb{H}^{\mThe} [Z] -  \mathbb{H}^{\mK, \mLam, \mThe} [X, Y, Z]  + \mathbb{H}^\mThe [Z]
\end{align*}
\endgroup
Thus, $\mathbb{I}^{\mK, \mThe}[X ; Z] + \mathbb{I}^{\mK, \mLam, \mThe}[X ; Y | Z] = \mathbb{H}^\mK [X] + \mathbb{H}^{\mLam, \mThe} [Y, Z] - \mathbb{H}^{\mK, \mLam, \mThe} [X, Y, Z] =\mathbb{I}^{\mK, \mLam, \mThe}[X ; Y, Z].$
\end{proof}{}

\newpage
\begin{theorem}  \emph{(\textbf{Data Processing Inequality})}$^\clubsuit$. \\
	If $X \rightarrow Y \rightarrow Z$ is a Markov chain,  $ \mathbb{I}^{\mK, \mThe}[X ; Z]  \le \mathbb{I}^{\mK, \mLam}[X ; Y] +  \mathbb{I}^{\mK, \mThe, \mLam}[X ; Z | Y].$
\end{theorem}
\begin{proof}
\begin{align*}
    & \mathbb{I}^{\mK, \mLam, \mThe}[X ; Y, Z]  
      \stackrel{\text{(Thm. \ref{thm:mi_chain})}}{=} \mathbb{I}^{\mK, \mLam}[X ; Y] + \mathbb{I}^{\mK, \mLam, \mThe}[X ; Y | Z] 
     \stackrel{\text{(Thm. \ref{thm:mi_chain})}}{=}  \mathbb{I}^{\mK, \mThe}[X ; Z] + \mathbb{I}^{\mK, \mThe, \mLam}[X ; Z | Y].
\end{align*}
Therefore $\mathbb{I}^{\mK, \mLam}[X ; Z]  + \mathbb{I}^{\mK, \mLam, \mThe}[X ; Y | Z] = \mathbb{I}^{\mK, \mThe}[X ; Z] + \mathbb{I}^{\mK, \mThe, \mLam}[X ; Z | Y].$ Finally, we have that $\mathbb{I}^{\mK, \mLam, \mThe}[X ; Y | Z] \ge 0$, which in turn implies that $\mathbb{I}^{\mK, \mLam}[X ; Z]  \le \mathbb{I}^{\mK, \mLam}[X; Y] + \mathbb{I}^{\mK, \mThe, \mLam}[X; Z|Y]. $
\end{proof}

\setcounter{theorem}{6}
Additionally, we are able to prove a series of inequalities illuminating the influence of the similarity matrix on joint entropy in extreme cases:
\begin{theorem}
	For any similarity kernels $\kappa$ and $\lambda$,
	$\mathbb{H}^{\mK}[X] = \mathbb{H}^{\mK \otimes \mat{J}}[X, Y] \le   \mathbb{H}^{\mK \otimes \mLam}[X,Y]  \le \mathbb{H}^{\mK \otimes \eye}[X, Y] = \mathbb{H}^{\eye}[Y] + \mathbb{H}^{\mK, \eye}[X | Y]$
	\label{thm:ent_ineq}
\end{theorem}
\begin{proof}
The first result, $\mathbb{H}^{\mK}[X] = \mathbb{H}^{\mK \otimes \mat{J}}[X, Y]$ follows by noting that $\lambda(y, y') = 1$ for all $y, y'$:
\begin{align*}
\mathbb{H}^{\mK \otimes \mat{J}}[X, Y] &=\Exp_{x, y} \log \lspar \Exp_{x', y'} \kappa(x, x')\lambda(y, y') \rspar=\Exp_{x} \log \lspar\Exp_{x'} \kappa(x, x') \rspar =\mathbb{H}^{\mK}[X]
\end{align*}
 $\mathbb{H}^{\mK \otimes \mat{J}}[X, Y] \le   \mathbb{H}^{\mK \otimes \mLam}[X,Y] \le \mathbb{H}^{\mK \otimes \eye}[X, Y]$ follows by monotonicity of the entropy in the similarity matrices.
 
  $\mathbb{H}^{\mK \otimes \eye}[X, Y] = \mathbb{H}^{\eye}[Y] + \mathbb{H}^{\mK, \eye}[X | Y]$ follows by the chain rule of conditional entropy.
\end{proof}{}

	\section{Verifying the concavity of $\KOEnt[\cdot]$}\label{sec:concavity_details}

\subsection{Proof attempts}

We have made several attempts to show that the GAIT entropy is a concave function at $\alpha=1$. As this is a critical component in our theoretical developments, we provide a list of our previously unsuccessful approaches, in the hopes of facilitating the participation of interested researchers in answering this question.

\begin{itemize}
    \item Jensen's inequality for the $\log(\Kp)$ or $\log(\Kq)$ terms is too loose.
    \item The bound $\log b \le \frac{b}{a} + \log (a) - 1$ applied to the ratio $\frac{\Kp}{\Kq}$ results in a loose bound.  
    \item $-p \log(p)$ is known to be a concave function. However, the action of the similarity matrix on the distribution inside the logarithmic factor in $-\vp^T \log (\Kp)$ complicates the analysis.
    \item The Donsker-Varadhan representation of the Kullbach-Leibler divergence goes in the wrong direction and adds extra terms.
    \item Bounding a Taylor series expansion of the gap between the linear approximation of an interpolation and the value of the entropy along the interpolation.  The analysis is promising but becomes unwieldy due to the presence of $\frac{\mK \vq}{\mK \vp}$ terms.
\end{itemize}

\subsection{Positive definiteness of the Hessian of the negative entropy}

Straightforward computation based on the definition of the GAIT entropy leads to a remarkably simple form for the Hessian of the negative entropy. 
\begin{theorem}
    \begin{equation*}
      -\nabla^2_\mathbb{P} [ \mathbb{H}_1^{\mathbf{K}}[\mathbb{P} ]] = \mathbf{K} \text{diag} \lpar \frac{1}{ \mathbf{K} \mathbb{P}} \rpar - \mathbf{K} \text{diag} \lpar \frac{\mathbb{P}}{ \lpar \mathbf{K} \mathbb{P} \rpar^2} \rpar + \text{diag} \lpar \frac{1}{ \mathbf{K} \mathbb{P}} \rpar  \mathbf{K}
    \end{equation*}
    
Moreover, $-\nabla^2_\mathbb{P} [ \mathbb{H}_1^{\mathbf{K}}[\mathbb{P} ]]$ is positive definite in the $2 \times 2$ case.
\end{theorem}

The proving of the conjecture is equivalent proving positive definiteness of the matrix presented above. Furthermore, since we are interested in the behavior of the GAIT entropy operating on probability distributions, it is even sufficient to only consider the action of this matrix as a quadratic form the set of mass-preserving vectors with entries adding up to zero. 


\subsection{Numerical experiments}


\textbf{Random search on $\mathbb{D}^\mK[\vp||\vq]\geq 0$.} We perform a search over vectors $\vp$ and $\vq$ drawn randomly from the simplex, and over random positive definite similarity Gram matrices $\mK$. We have tried restricting our searches to $\vp$ and $\vq$ near the center of the simplex and away from the center, and to $\mK$ closer to $\eye$ or $\mathbf{J}$. In every experiment, we find that $\mathbb{D}^\mK[\vp||\vq]\geq 0$.

Consider the wide experimental setup for search defined in Tab.~\ref{tab:random_search}. Fig.~\ref{fig:concave_histogram} shows the histogram of  $\mathbb{D}^\mK$ over this search, empirically showing the non-negativity of the divergence, and, thus the concavity of the GAIT entropy.

\begin{minipage}{\textwidth}
  \begin{minipage}[b]{0.49\textwidth}
    \centering
        \captionof{table}{Experimental setup for random search.}
        \begin{tabular}{|c|c|}
        \hline
        \textbf{Quantity} & \textbf{Sampling process} \\ \hline
        $n \in \mathbb{Z}$ & $n \sim \text{Uniform}(\{2, \ldots, 11\})$ \\ \hline
        $\boldsymbol{\gamma} \in \mathbb{Z}^{n\times n}$ & $\gamma_{i,j} \sim \text{Uniform}(\{0,\ldots,9\})$ \\ \hline
        $\mathbf{L} \in \mathbb{R}^{n\times n}$ & $L_{i,j} \sim \text{Uniform}(0, 1)^{\gamma_{i,j}}$ \\ \hline
        $\mK \in \mathbb{R}^{n\times n}$ & $\mK = \text{min}(1, \eye + \mathbf{L}\mathbf{L}^T/n)$ \\ \hline
        $\boldsymbol{\alpha} \in \mathbb{R}^n$ & $\alpha_i \sim \text{Uniform}(0, 10)$ \\ \hline
        $\boldsymbol{\beta} \in \mathbb{R}^n$ & $\beta_i \sim \text{Uniform}(0, 10)$ \\ \hline
        $\vp \in \mathbf{\Delta}_n$ & $\vp \sim \text{Dirichlet}(\boldsymbol{\alpha})$ \\ \hline
        $\vq \in \mathbf{\Delta}_n$ & $\vq \sim \text{Dirichlet}(\boldsymbol{\beta})$ \\ \hline
        \end{tabular}%
        \label{tab:random_search}
        \vspace{0.5cm}
    \end{minipage}
  \hfill
  \begin{minipage}[b]{0.49\textwidth}
    \centering
	\includegraphics[trim={0cm 0.5cm 0cm 0cm},clip,height=0.17\textheight]{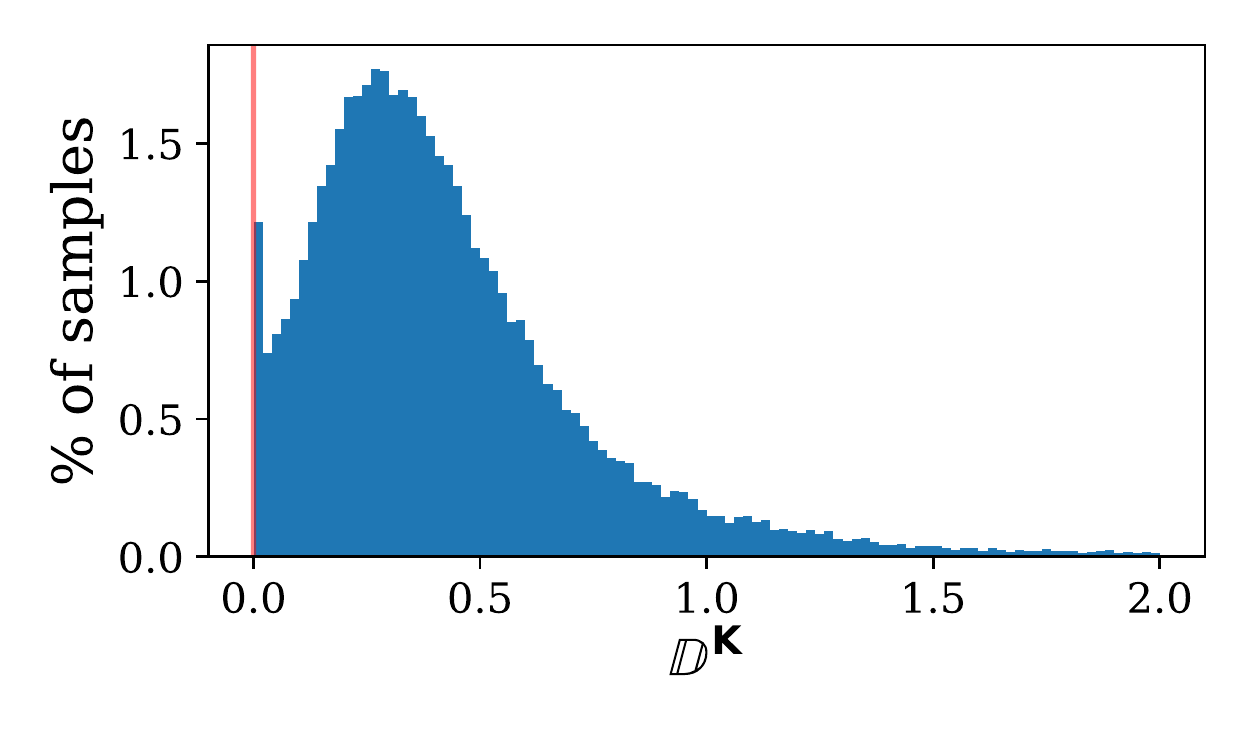}
    \captionof{figure}{Histogram of GAIT entropies obtained using the quantities sampled according to Tab.~\ref{tab:random_search}.} 
    \label{fig:concave_histogram}
  \end{minipage}
  
  \end{minipage}

\textbf{Random search on $-\nabla^2_\mathbb{P} [ \mathbb{H}_1^{\mathbf{K}}[\mathbb{P} ]]$.}

We empirically study the positive definiteness of this matrix via its spectrum. For this, we sample a set of $n$ points in $\mathbb{R}^d$ as well as a (discrete) distribution $\mathbb{P}$ over those points. Then we construct the Gram matrix induced by the kernel $\kappa(x, y) = \exp{(-||x - y||_p)}$. The location of the points, $\mathbb{P}$, $n$, $d$ and $p>=1$ are sampled randomly.  

We performed extensive experiments under this setting and never encountered an instance such that $-\nabla^2_\mathbb{P} [ \mathbb{H}_1^{\mathbf{K}}[\mathbb{P} ]]$ would have any negative eigenvalues. We believe this experimental setting is more holistic than the above experiments since it considers the whole spectrum of the (negative) Hessian rather than a “directional derivative” towards another sampled distribution $\mathbb{Q}$.

\textbf{Optimization.} As an alternative to random search, we also use gradient-based optimization on $\vp$, $\vq$ and $\mK$ to minimize $\mathbb{D}^\mK[\vp||\vq]$. Starting from random initializations, our objective function always converges to values very close to (yet above) zero.


Furthermore, freezing $\mK$ and optimizing over either $\vp$ or $\vq$ while holding the other fixed, results in $\vp = \vq$ at convergence.  On the other hand, if $\vp$ and $\vq$ are fixed such that $\vp\neq\vq$, optimization over $\mK$ converges to $\mK = \mathbf{J}$. We note from the definition of the GAIT divergence that when $\vp=\vq$ or $\mK=\mathbf{J}$, $\mathbb{D}^\mK[\vp||\vq]=0$, which matches the value we obtain at convergence when trying to minimize this quantity.

Recall that the experiments presented in Sec.~\ref{sec:experiments} involve the minimization of some GAIT divergence. We never encountered a negative value for the GAIT divergence during any of these experiments.

\subsection{Finding maximum entropy distributions with gradient ascent}
\label{sec:maxent}

An algorithm with an exponential run-time to find \emph{exact} maximizers of the entropy $\KEnt[\cdot]$ is presented in \cite{leinster2016}. We exploit the fact that the objective is amenable to gradient-based optimization techniques and conduct experiments in spaces with thousands of elements. This also serves as an empirical test for the conjecture about the concavity of the function: there must be a unique maximizer for $\mathbb{H}^\mathbf{\mK}_1[\cdot]$ if it is concave.

We test our ability to find distributions with maximum GAIT entropy via gradient descent. We sample 1000 points in dimensions 5 and 10, and construct a similarity space using a RBF kernel with $\sigma=1$. Then we perform 100 trials by setting the logits of the initialization using a Gaussian distribution with variance 4 for each of the 1000 logits that describe our distribution. We use Adam with learning rate 0.1 and $\alpha=1$. The optimization results are shown in Fig. \ref{fig:maxent}. We reliably obtain negligible variance in the objective value at convergence across random initializations, thus providing an efficient alternative for finding approximate maximum-entropy distributions.

\begin{figure}[h]
	\centering
	\includegraphics[height=0.15\textheight]{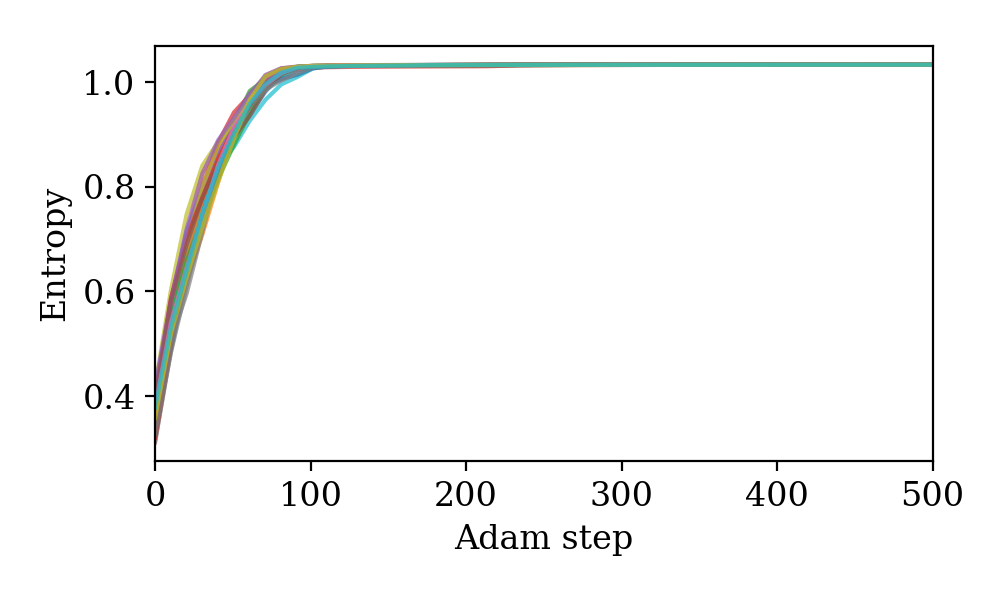}
	\includegraphics[height=0.15\textheight]{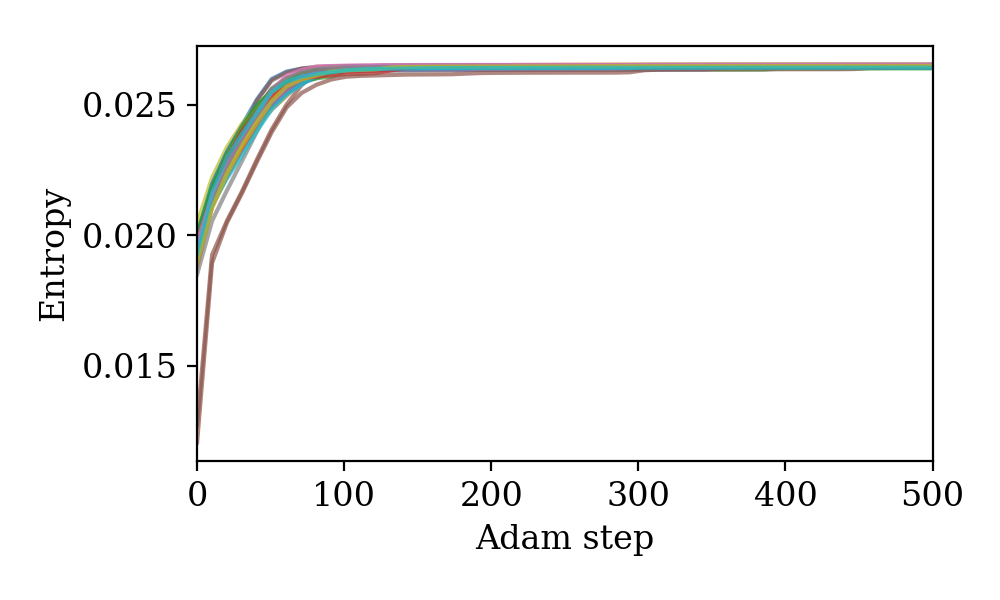}\\
	\vspace{-1ex}
	\caption{Optimization curves for measures with support 1000 in dimension 5 (left) and 100 (right).} 
	\label{fig:maxent}
\end{figure}

\section{Interpolation and Approximation} 

\label{sec:interp}

In all experiments for Figs. \ref{fig:atomic_approx}-\ref{fig:mnist_barys}, we minimize the GAIT divergence using  AMSGrad \citep{amsgrad} in PyTorch \citep{pytorch}. We parameterize the weights of empirical distributions using a softmax function on a vector of temperature-scaled logits. All experiments in the section are run on a single CPU.

\subsection{Approximating measures with finite samples}

In Fig. \ref{fig:atomic_approx} we optimize our approximating measure using Adam for 3000 steps with a learning rate of $10^{-3}$ and minibatches constructed by sampling 50 examples at each step.  We use a Gaussian kernel with $\sigma=0.02$.

In Fig. \ref{fig:supersamples}, we approximate a continuous measure with an empirical measure supported on 200 atoms. We execute Adam for 500 steps using a learning rate of $0.05$ and minibatches of 100 samples from the continuous measure to estimate the discrepancy. The similarity function is given by a polynomial kernel with exponent 1.5: $\kappa(\vect{x}, \vect{y}) \triangleq \frac{1}{1 + \norm{\vect{x} - \vect{y}}^{1.5}}$. Fig. \ref{fig:super_comp} shows that we achieve results of comparable quality to those of \citet{stoch_w_barys} 

\begin{figure}[h]
	\centering
	\includegraphics[height=0.17\textheight]{imgs/bregman_results/supersamples_data_model_tight.png}
	\includegraphics[height=0.17\textheight]{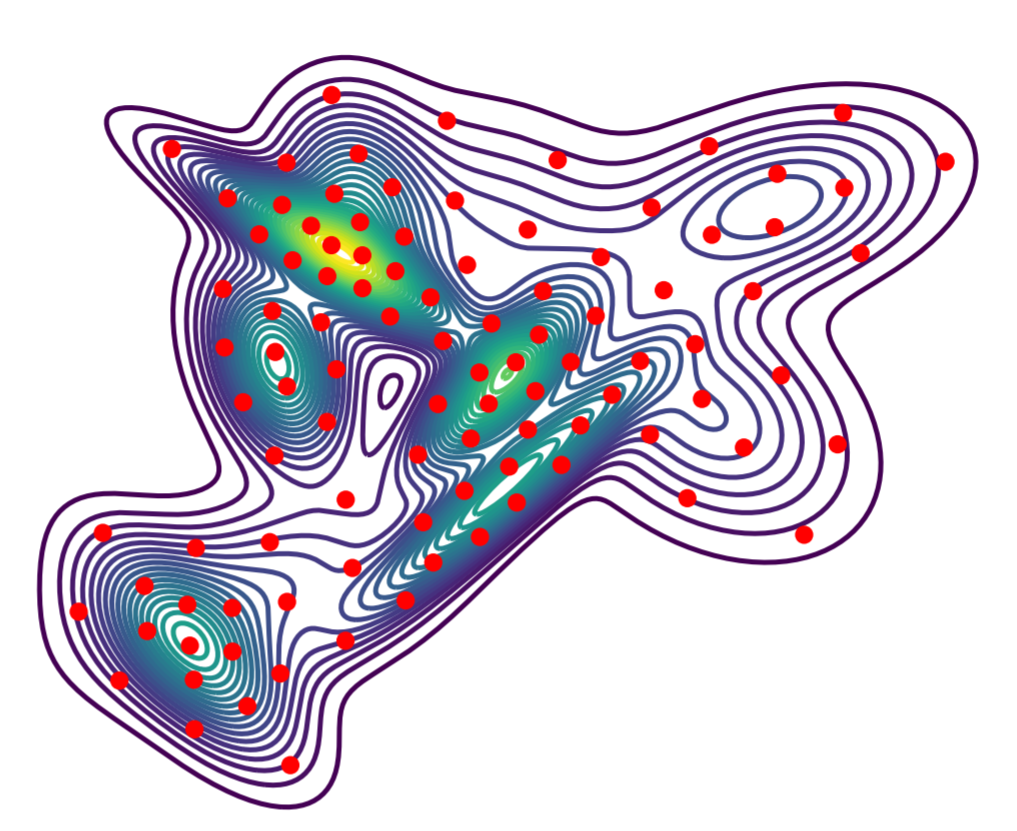}
	\includegraphics[height=0.17\textheight]{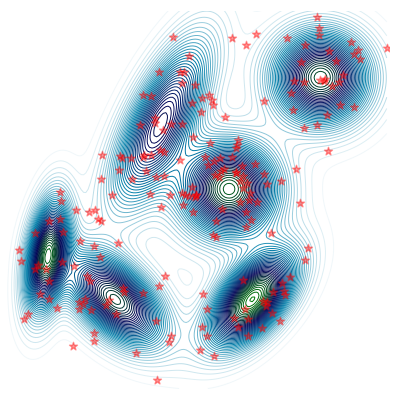}
	\caption{\textbf{Left and Center:} Approximation of a mixture of Gaussians density using our method and the proposal of \cite{stoch_w_barys} (taken from paper). \textbf{Right:} i.i.d samples from the real data distribution.} 
	\label{fig:super_comp}
\end{figure}

\subsection{Image barycenters}
We compute barycenters for each class of MNIST and Fashion-MNIST. We perform gradient descent with Adam using a learning rate of $0.01$ with minibatches of size 32 for 500 optimization steps.  We use a Gaussian kernel with $\sigma=0.04$. The geometry of the grid on which images are defined is given by the Euclidean distance between the coordinates of the pixels. In Fig. \ref{fig:bary_comparison}, we provide barycenters for the each of classes of Fashion MNIST computed via a combination of the methods of \cite{bregman} and \cite{cuturi_fast_bar}.

\begin{figure}[h]
    \centering
    \includegraphics[width=0.8\textwidth]{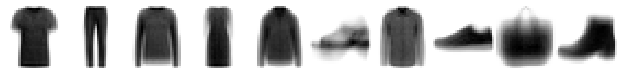}
    \caption{Barycenters for Fashion MNIST computed using our method.}
    \label{fig:bary_comparison}
\end{figure}

\subsection{Text summarization} \label{sec:text}
For our text example, we use the article from the STAT-MT parallel news corpus titled 
``Why Wait for the Euro?'', by Leszek Balcerowicz.  The full text of the article can be found at \url{https://pastebin.com/CnBgbpsJ}.  We use the 300-dimensional GLoVe vectors found at \url{http://nlp.stanford.edu/data/glove.6B.zip} as word embeddings.  TF-IDF is calculated over the entire English portion of the parallel news corpus using the implementation in Scikit-Learn \citep{sklearn}.  We filter stopwords based on the list provided by the Natural Language Toolkit \citep{nltk}. To encourage sparsity in the approximating measure $\vect{q}$, we add the $0.75$-norm of $\vect{q}$ to the divergence loss, weighted by a factor of $0.01$.  We optimize the loss with gradient descent using Adam optimizer, with hyperparameters $\beta_1 = 0, \beta_2 = 0.9, \text{ learning rate}=0.001$, for 25,000 iterations.  Since a truly sparse $\vect{q}$ is not reachable using the softmax function and gradient descent, we set all entries $\vect{q}_i < 0.01$ to be 0 and renormalize after the end of training.  $\vect{q}$ is represented by the softmax function, and is initialized uniformly.

We examine the influence of varying $\sigma$ in Fig. \ref{fig:wordclouds_sigma}.  Decreasing $\sigma$ leads to $\vect{K}$ approaching $\eye$, and the resulting similarity more closely approximates the original measure.  As $\sigma$ approaches 0.01, the two measures become almost identical. See Fig. \ref{fig:wordclouds_sigma}, bottom-left and bottom-right.

\begin{figure*}[h]
    \centering
	\includegraphics[width=0.3\columnwidth]{imgs/bregman_results/cloud_1_05.png} \hspace{1cm}
	\includegraphics[width=0.3\columnwidth]{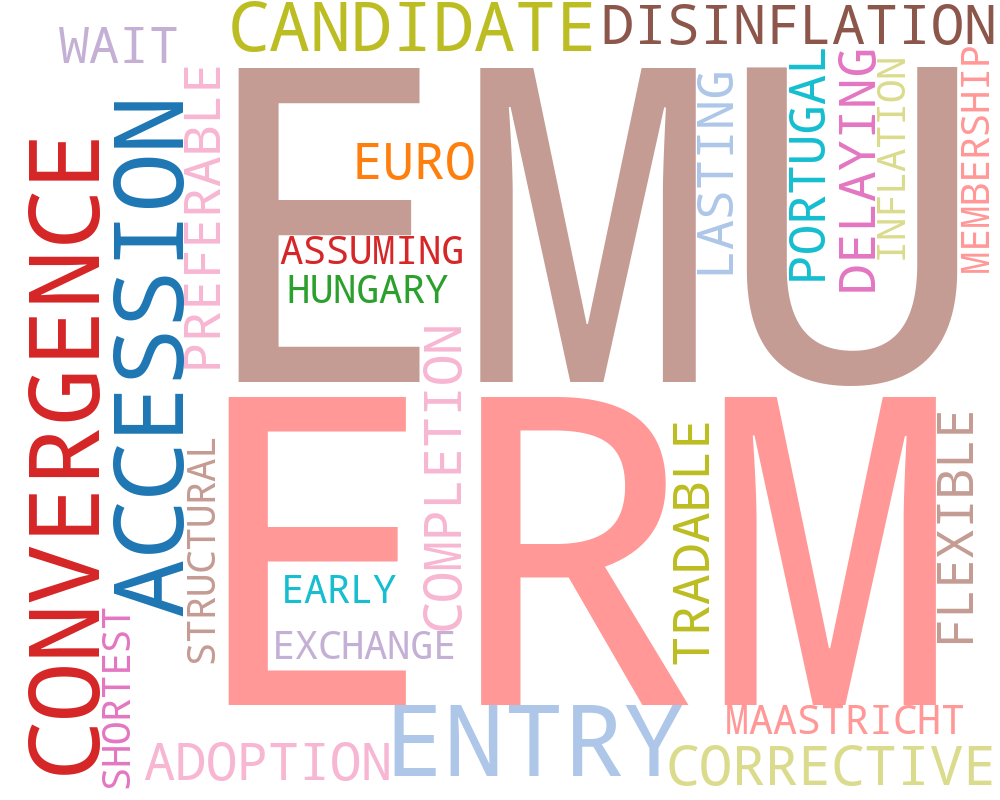} \\ \vspace{0.5cm}
	\includegraphics[width=0.3\columnwidth]{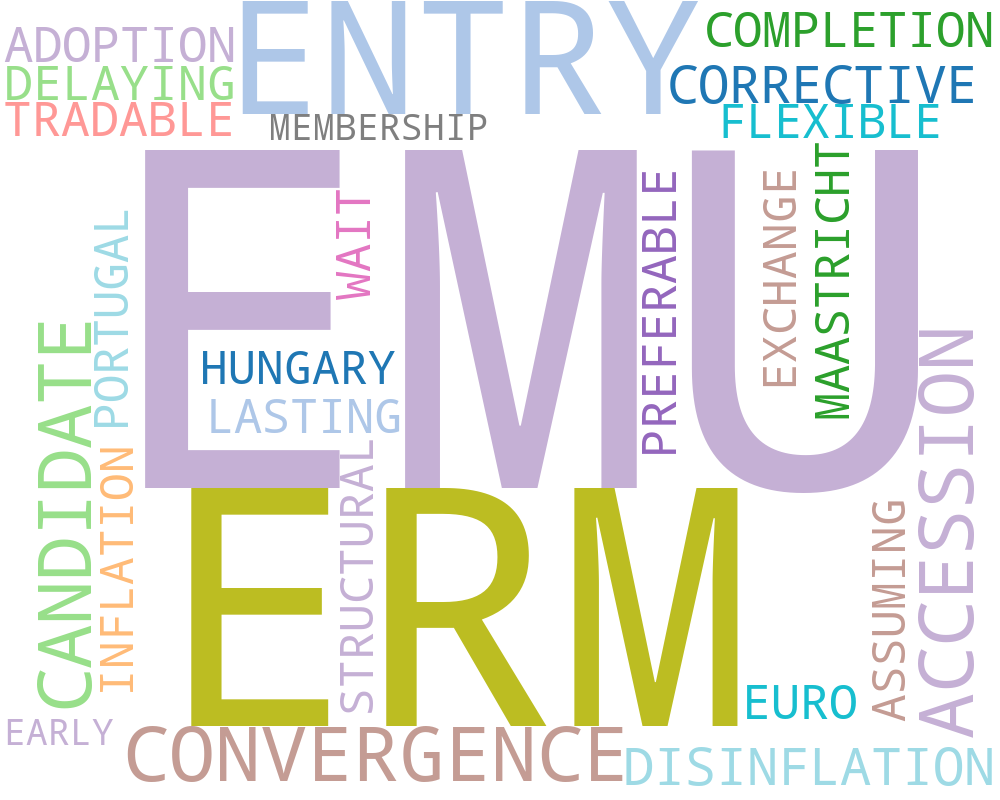} \hspace{1cm}
	\includegraphics[width=0.3\columnwidth]{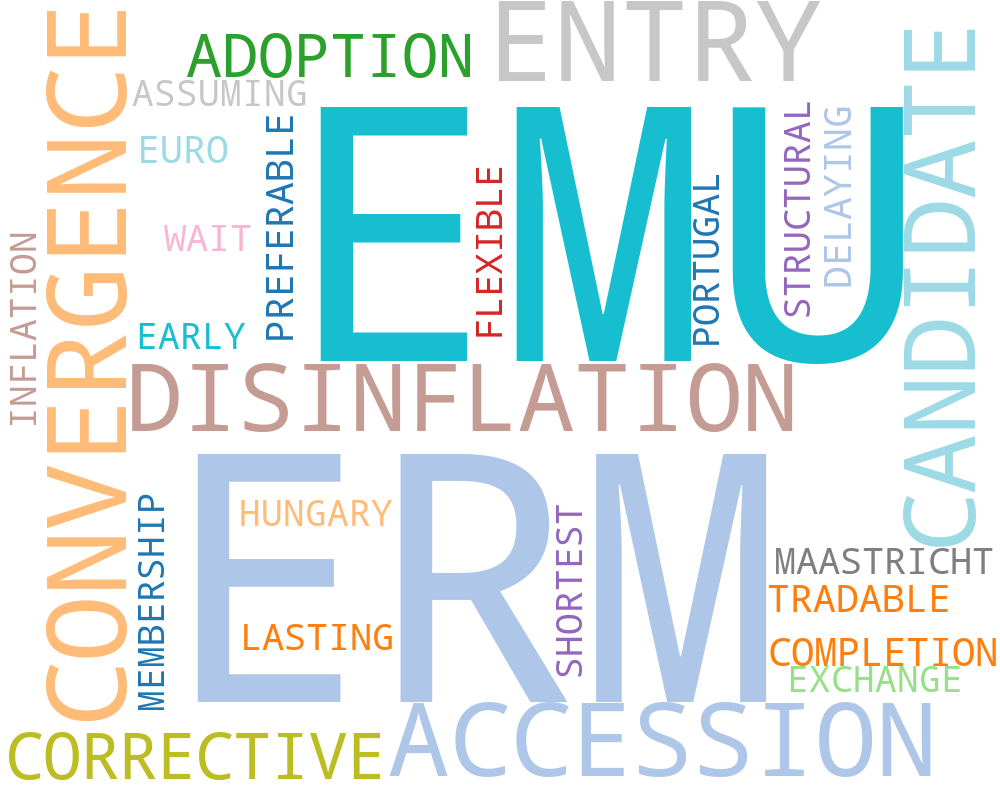}
	\caption{\textbf{Top-left:} Word cloud generated by our method at $\sigma=0.5$. \textbf{Top-right:} Word cloud generated by our method at $\sigma=0.1$. \textbf{Bottom-left:} Word cloud generated by our method at $\sigma=0.01$. \textbf{Bottom-right:} Top 43 original TF-IDF words.}
	\label{fig:wordclouds_sigma}
\end{figure*}

\section{GAN evaluation and mode counting}\label{sec:mode_details}
When the data available takes the form of many i.i.d. samples from a continuous distribution, a natural choice is to generate a Gram matrix $\vect{K}$ using a similarity measure such as an RBF kernel $\kappa_\sigma(\vx, \vect{y}) = \exp{\left(\frac{-||\vect{x}-\vect{y}||^2}{2\sigma^2}\right)}$. 

For comparison, we adapt the birthday paradox-based approach of \cite{birthday}.  Strictly speaking, their method requires human evaluation of possible duplicates, and is thus not comparable to our approach.  As such, we propose an automated version using the same assumptions. We define $\vect{x}$ and $\vect{y}$ as colliding when $d(\vect{x}, \vect{y}) < \epsilon$, and note that the expected number of collisions for a distribution with support $n$ in a sample of size $m$ is $c=\frac{m(m-1)}{n}$.  We can thus estimate $\hat{n} = \frac{m(m-1)}{c}$. When varying $\epsilon$, we observe behavior very similar to that of our entropy measure, with a plateau at $\hat{n} = C$ in our example of a mixture of $C$ Gaussians. The results of this comparison are presented in Fig. \ref{fig:mog}.

To test this on a more challenging dataset, we use a 2-dimensional representation for MNIST obtained using UMAP~\citep{umap}, shown in Fig. \ref{fig:mnist_count}.  Although our method no longer shows a clear plateau at $\mathbb{H}^{\mK}_1[\Prob] \approx \log 10 \approx 2.3$, it does transition from exponential to linear decay at approximately this point, which coincides with the point of minimum curvature with respect to $\sigma$, $\mathbb{H}^{\mK}_1[\Prob] \approx \log 10$. Similar behavior is observed in the case with birthday-inspired estimate; here the point of minimum curvature has $\hat{n} \approx 8$.

\begin{figure*}[h]
    \centering
    \vspace{-1ex}
    \includegraphics[trim={.5cm 0.8cm .5cm .5cm},clip,width=0.32\textwidth]{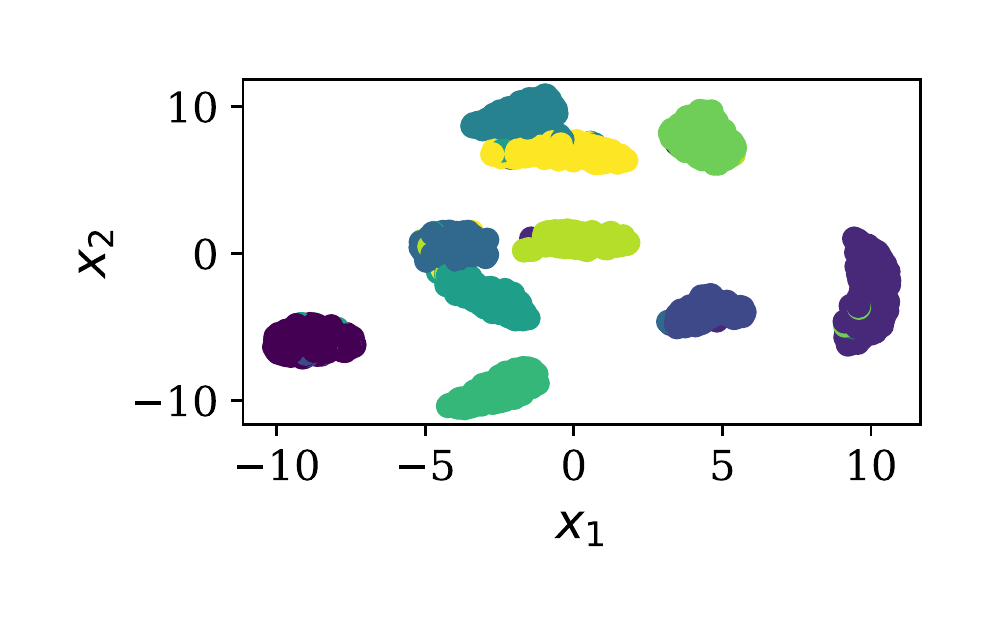} \hfill \includegraphics[trim={0cm 0.0cm .5cm .5cm},clip,width=0.32\textwidth]{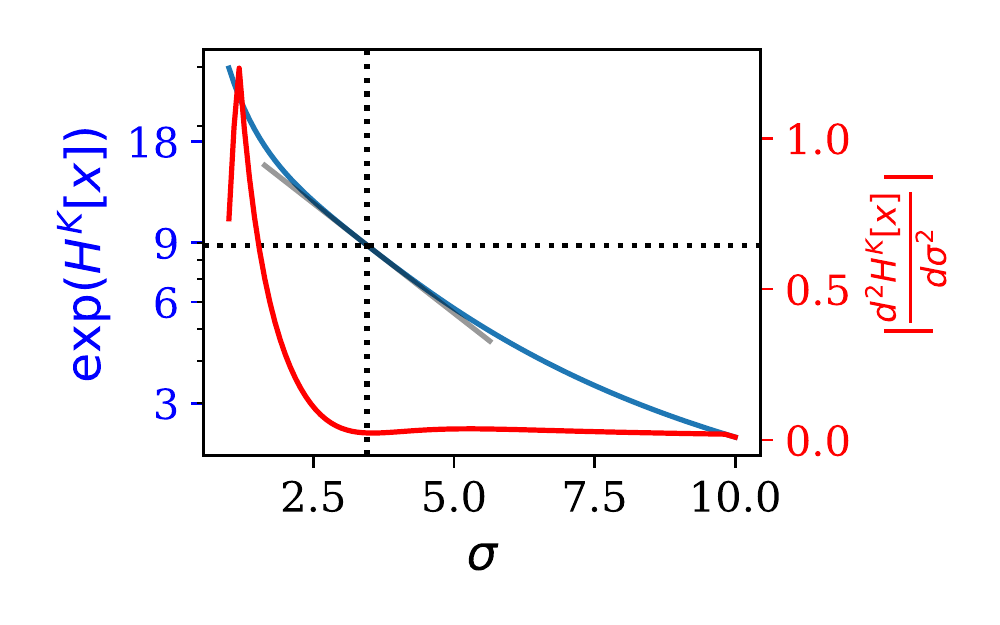}
    \hfill \includegraphics[trim={0cm 0.0cm .5cm .5cm},clip,width=0.32\textwidth]{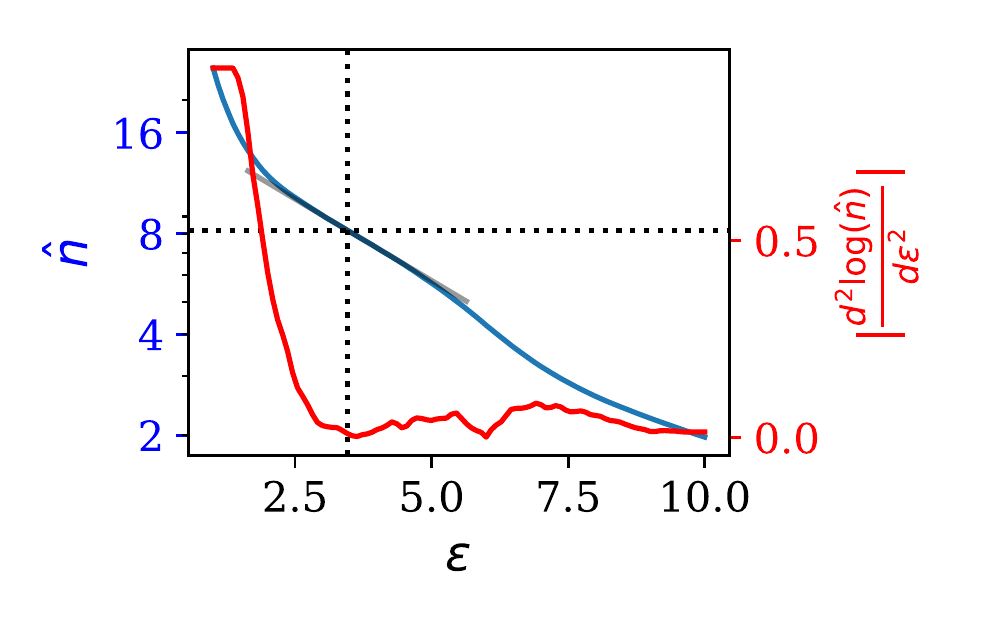}
    \caption{\textbf{Left:} A 2000-image subset of MNIST reduced to 2 dimensions by UMAP.  \textbf{Center:} Our mode estimation.  \textbf{Right:} The birthday paradox method estimate. Note that the left axis is logarithmic.}
    \label{fig:mnist_count}
\end{figure*}

Finally, we also apply this method to evaluating the diversity of GAN samples.  We train a simple WGAN \citep{wgan} on MNIST, and find that the assessed entropy increases steadily as training progresses and the generator masters more modes (see Fig. \ref{fig:gan_eval}). Note that the entropy estimate stabilizes once the generator begins to produce all 10 digits, but long before sample quality ceases improving.

\begin{figure*}[h]
    \centering
    \begin{minipage}{0.33\textwidth}
    \includegraphics[width=\textwidth]{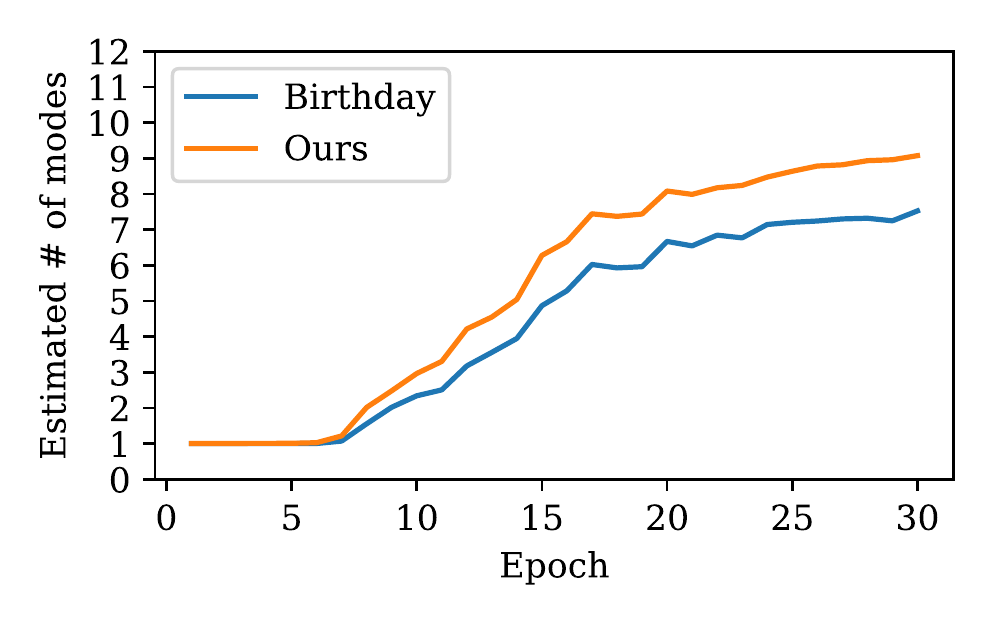}
    \end{minipage}
    \hfill
    \begin{minipage}{0.65\textwidth}
    \includegraphics[width=\textwidth]{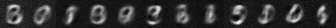}
    \includegraphics[width=\textwidth]{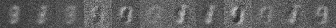}
    \includegraphics[width=\textwidth]{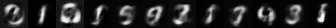}
    \end{minipage}
    \caption{\textbf{Left:} The estimated numbers of modes in the output of a WGAN trained on MNIST. \textbf{Right:} Samples from the same WGAN after 5, 15 and 25 epochs.}
    \label{fig:gan_eval}
\end{figure*}

In all of the experiments corresponding to mode counting, we use $\alpha=1$ and the standard RBF kernel $\kappa_\sigma(\vx, \vect{y}) = \exp{\left(\frac{-\norm{\vect{x}-\vect{y}}^2}{2\sigma^2}\right)}$.  Note that this differs from the kernel given in Section \ref{sec:theory} by using squared Euclidean distance rather than Euclidean distance.  To estimate the point with minimum curvature, we find the value of $\log \hat{n}$ or $\mathbb{H}^{\mK}_1[\vect{p}]$ at 100 values of $\sigma$ or $\epsilon$ evenly spaced between $0.1$ and $25$, and empirically estimate the second derivative with respect to $\sigma$ or $\epsilon$. In the case of the birthday estimate, which is not continuous on finite sample sizes, we use a Savitzky-Golay filter \citep{savgol} of degree 3 and window size 11 to smooth the derivatives.  We estimate the point of minimum curvature to be the first point when the absolute second derivative passes below $0.01$.

To evaluate GANs, we train a simple WGAN-GP~\citep{wgan-gp} with a 3-hidden-layer fully-connected generator, using the ReLU nonlinearity and 256 units in each hidden layer, on a TITAN Xp GPU.  Our latent space has 32 dimensions sampled i.i.d. from $\mathcal{N}(0, 1)$ and the discriminator is trained for four iterations for each generator update.  We use the Adam with learning rate $10^{-4}$ and $\beta_1 = 0$, $\beta_2 = 0.9$.  The weight of the gradient penalty in the WGAN-GP objective is set to $\lambda = 10$. 

To count the number of modes in the output of the generator, we use an instance of UMAP fitted to the entire training set of MNIST to embed all input in $\mathbb{R}^2$.  We use 1,000 samples of true MNIST data to estimate values of $\sigma$ (for our entropy method) and $\epsilon$ for the birthday paradox-based method that minimize curvature and yield estimates of $\exp{\mathbb{H}^{\mK}_1[\vect{p}]} \approx 10$ and $\hat{n} \approx 10$.  We then apply these methods to the output of the generator after each of the first 30 epochs, and report the resulting $\hat{n}$ or $\exp{\mathbb{H}^{\mK}_1[\vect{p}]}$.

\newpage

\section{Generative models}\label{sec:app_genmodels}

For all the generative models in Section~\ref{sec:comparison_ot}, we employ an experimental setup similar to the setup used by \citet{cuturi_learning} for learning generative models on MNIST. Thus, our generative model is a $2$-layer multilayer perceptron with one hidden layer of 500 dimensions with ReLU non-linearities, using a $2$D latent space, trained using mini-batches of size $200$. Note that their method requires a batch size of $200$ to get reasonable generations, but we also obtain comparable results with a significantly smaller batch size of $50$. Since \citet{cuturi_learning} sample latent codes from a unit square, we do the same for MNIST here for easy comparison but sample from a standard Gaussian for Swiss roll and Fashion-MNIST datasets. We train our models by minimizing $\BDiv [ \hat{\Prob} \, || \, \hat{\Qrob} ]$, where $\hat{\Prob}$ is the target empirical measure and $\hat{\Qrob}$ is the model. $\mK$ is the Gram matrix corresponding to a RBF kernel with $\sigma=0.2$ for Swiss roll data, and $\sigma=1.6$ for MNIST and Fashion-MNIST. We use Adam with a learning rate of $5\times 10^{-4}$ to train our models. Fig.~\ref{fig:sinkhorn_compare} compares the manifolds learned by minimizing our divergence with batch sizes $200$ and $50$ with that learned by minimizing the Sinkhorn loss \citep{cuturi_learning} for MNIST.

\begin{figure*}[h]
    \centering
    \includegraphics[width=0.3\textwidth]{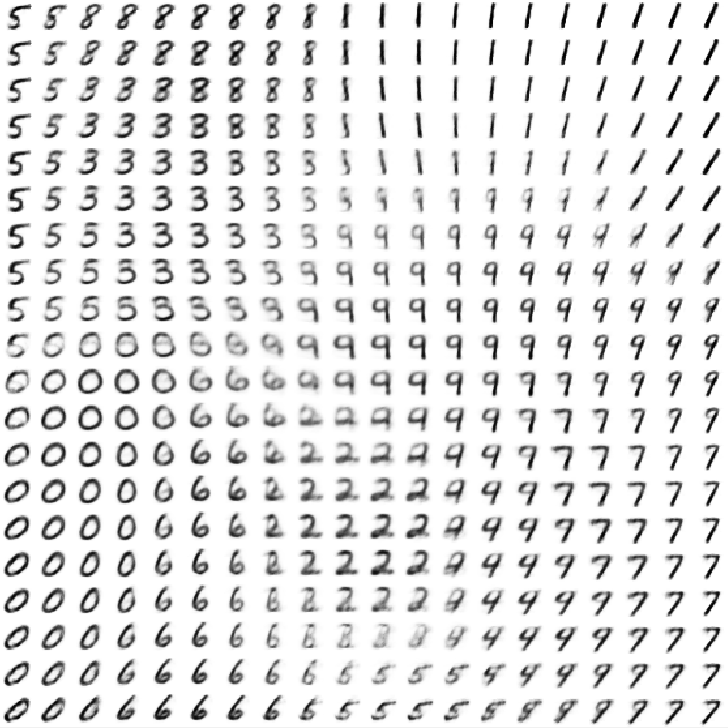}
    \hfill
    \includegraphics[width=0.3\textwidth]{imgs/bregman_results/sinkhorn_compare.png}
    \hfill
    \includegraphics[width=0.3\textwidth]{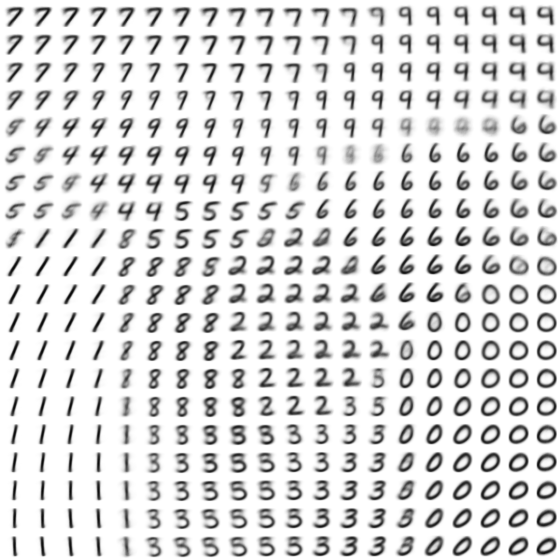}
    \caption{\textbf{Left:} Manifold learned by minimizing Sinkhorn loss, taken from \citet{cuturi_learning}. \textbf{Center:} Manifold learned by minimizing GAIT divergence using their experimental setup. \textbf{Right:} Manifold learned by minimizing GAIT divergence with batch size $50$.}
    \label{fig:sinkhorn_compare}
\end{figure*}

We further compare our generations with those done by variational auto-encoders \citep{vae}. Following their setup,  we use $\tanh$ as the non-linearity in the $2$-layer multilayer perceptron and a lower batch size of $100$, along with the latent codes sampled from a standard Gaussian distribution. We compare our results with theirs in Fig.~\ref{fig:vae_compare}. Both figures are generated using latent codes obtained by taking the inverse c.d.f. of the Gaussian distribution at the corresponding grid locations, similar to the work of \citet{vae}.

\begin{figure*}[h]
    \centering
    \includegraphics[width=0.31\textwidth]{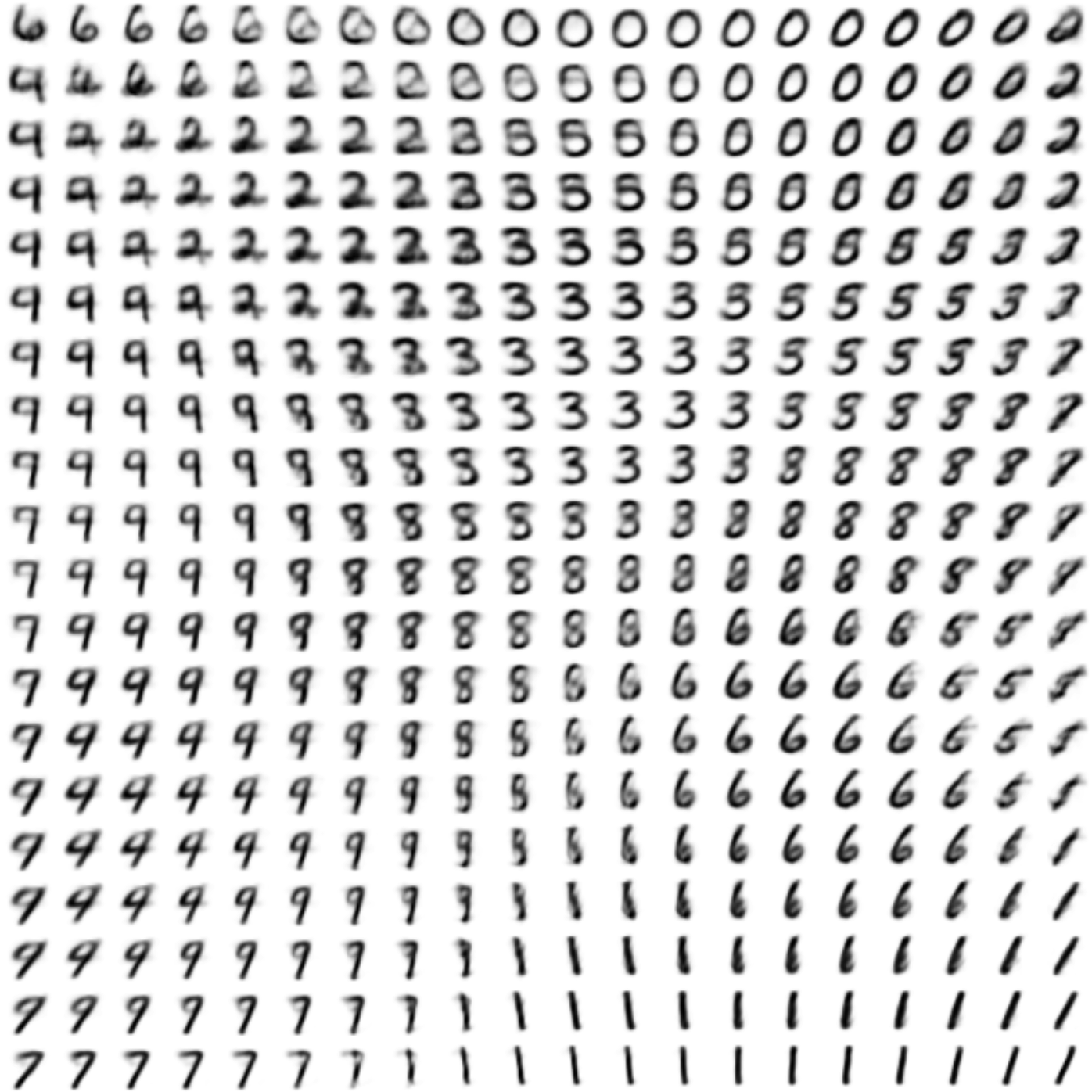}
    \hspace{1cm}
    \includegraphics[width=0.31\textwidth]{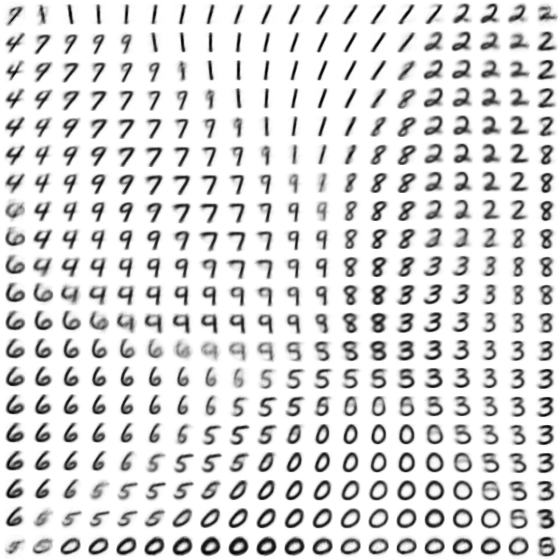}
    \hspace{0.7cm}
    \caption{\textbf{Left:} Manifold learned by Variational Autoencoder, taken from \citet{vae}. \textbf{Right:} Manifold learned by minimizing GAIT divergence using their experimental setup.}
    \label{fig:vae_compare}
\end{figure*}

Finally, in Fig.~\ref{fig:genmodel_samples}, we illustrate Fashion-MNIST and MNIST samples generated by our generative model with a $20$D latent space. The quality of our generations with a $20$D latent space is comparable to the samples generated by the variational auto-encoder with the same latent dimensions in \citet{vae}.

\begin{figure*}[h]
    \centering
    \includegraphics[width=0.31\textwidth]{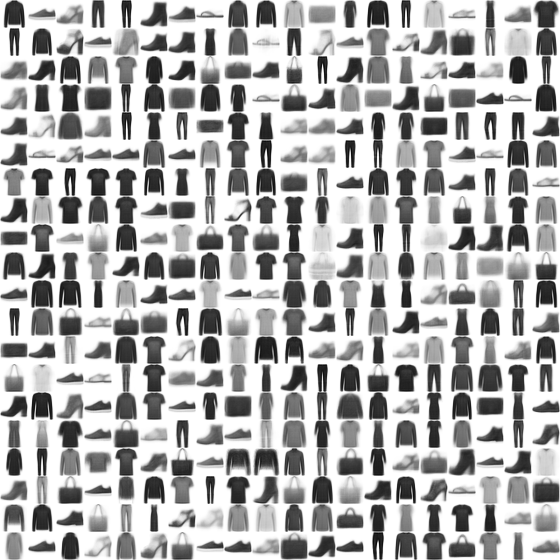}
    \hfill
    \includegraphics[width=0.31\textwidth]{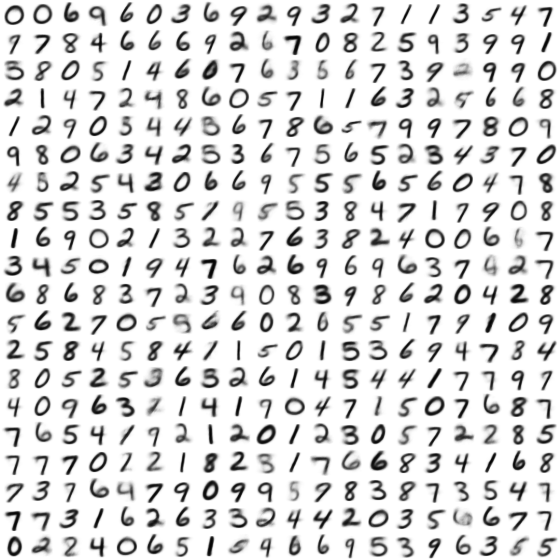}
    \hfill
    \includegraphics[width=0.18\textwidth]{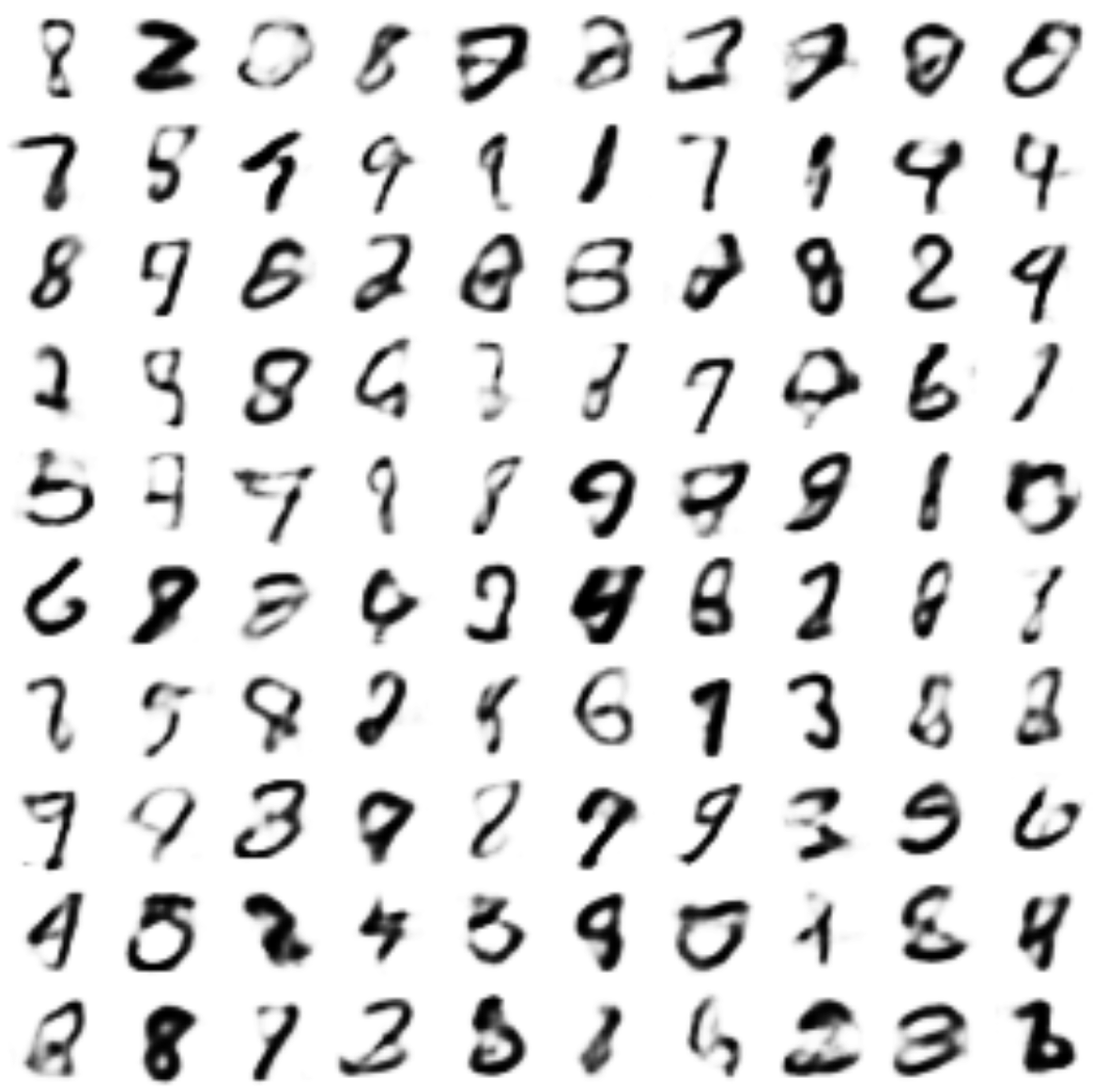}

    \caption{\textbf{Left:} Fashion-MNIST samples from our model with $20$D latent space. \textbf{Center:} MNIST samples from our model with $20$D latent space. \textbf{Right:} MNIST samples from Variational Autoencoder with $20$D latent space, picture taken from \citep{vae}.}
    \label{fig:genmodel_samples}
\end{figure*}




\section{Computational complexity}\label{sec:app_cplx}

\citet{conv_ot} shows how the computation of $\mathbf{K}\mathbb{P}$ can be efficiently performed using convolutions in the case of image-like data. For $d \times d$ images, this takes time $\mathcal{O}(d^3)$, instead of $\mathcal{O}(d^4)$ using a naive approach. Sinkhorn regularized optimal transport requires performing this computation this computation $L$, which highlights the value of the work of \citet{conv_ot} for applications with large $d$. The complexity for computing the close-form GAIT divergence is thus $\mathcal{O}(d^3)$, and the cost for approximating solving the optimal transport problem via Sinkhorn iterations is $\mathcal{O}(Ld^3)$. We draw the attention of the reader to the distinction between the width $d$ of the image, and the size of the support of the measures, $n=d^2$.

Fig. \ref{fig:gait_timing} shows compares the time required by the convolutional approaches of the GAIT divergence computation and the Sinkhorn algorithm approximating the Sinkhorn divergence, between two images of size $d \times d$. \citet{cuturi_learning} found $L=100$ necessary to perform well on generative modeling. Even for the comparatively low values of $L$ presented in Fig. \ref{fig:gait_timing}, we observe that the computation of the GAIT divergence is significantly faster than that of the approximate Sinkhorn divergence. It is possible to compute the GAIT divergence between two images of one megapixel in a quarter of a second (horizontal line). 

\begin{figure*}[h]
    \centering
    \includegraphics[scale=0.6]{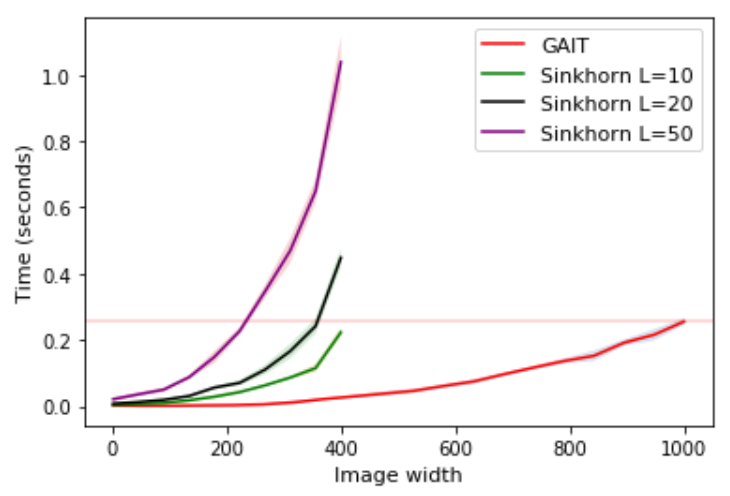}
    \caption{Time comparison between the computation of the GAIT and Sinkhorn divergences between randomly generated images of varying size. Error bars correspond to one standard deviation over a sample of size 30.}
    \label{fig:gait_timing}
\end{figure*}

\end{document}